\documentclass{article}
\usepackage{fullpage}

\usepackage[utf8]{inputenc} 
\usepackage[T1]{fontenc}    
\usepackage{hyperref}       
\usepackage{url}            
\usepackage{booktabs}       
\usepackage{amsfonts}       
\usepackage{nicefrac}       
\usepackage{microtype}      
\usepackage{tcolorbox}
\usepackage{bookmark}
\usepackage{enumerate}
\usepackage{enumitem}
\usepackage{color}
\usepackage{array}
\usepackage{booktabs}
\usepackage{natbib}
\hypersetup{
	colorlinks = true,
	citecolor = blue,
	linkcolor = blue
}

\usepackage{graphicx}
\usepackage{caption}
\usepackage{subcaption}

\usepackage{amsmath}
\usepackage{amsthm}
\usepackage{amssymb}
\usepackage{tikz}
\usepackage{bbm}
\usepackage{mathtools}
\usepackage{tablefootnote}
\usepackage{multirow}
\usepackage{xcolor}
\usetikzlibrary{arrows}
\allowdisplaybreaks[4]

\usepackage{mathrsfs}

\usepackage{algorithm}
\usepackage{algorithmic}
\usepackage{bm,todonotes}

\def\tod{\overset{d}{\to}}

\def\indep{\perp \!\!\! \perp}

\allowdisplaybreaks
\newtheorem{thm}{Theorem}[section]
\newtheorem{lem}{Lemma}[section]
\newtheorem{cor}{Corollary}[section]
\newtheorem{prop}{Proposition}[section]
\newtheorem{asmp}{Assumption}[section]

\newtheorem{rem}{Remark}[section]
\newtheorem{example}{Example}[section]

\def\1{\bm{1}}

\DeclareMathAlphabet{\mathsfit}{\encodingdefault}{\sfdefault}{m}{sl}
\SetMathAlphabet{\mathsfit}{bold}{\encodingdefault}{\sfdefault}{bx}{n}

\def\sG{{\mathbb{G}}}

\def\0{{\bf 0}}
\def\1{{\bf 1}}

\def\AM{{\mathcal A}}

\def\EM{{\mathcal E}}

\def\NM{{\mathcal N}}
\def\OM{{\mathcal O}}

\def\SM{{\mathcal S}}

\def\XM{{\mathcal X}}

\def\ZM{{\mathcal Z}}

\def\RB{{\mathbb R}}
\def\EB{{\mathbb E}}

\def\vect{\mathsf{vec}}

\DeclareMathOperator{\Tr}{Tr}

\title{
	Statistical Estimation of Confounded Linear MDPs: An Instrumental Variable Approach
}

\author{
	Miao Lu\thanks{School of Mathematical Sciences, Peking University; email: \texttt{lumiao@stu.pku.edu.cn}. } 
	\and
	Wenhao Yang\thanks{Academy for Advanced Interdisciplinary Studies, Peking University; email: \texttt{yangwenhaosms@pku.edu.cn}. } 
	\and
	Liangyu Zhang\thanks{Academy for Advanced Interdisciplinary Studies, Peking University; email: \texttt{zhangliangyu@pku.edu.cn}. } 
	\and
	Zhihua Zhang\thanks{School of Mathematical Sciences, Peking University; email: \texttt{zhzhang@math.pku.edu.cn}. }
}

\begin{document}

\maketitle

\begin{abstract}
	In an Markov decision process (MDP), unobservable confounders may exist and have impacts on the data generating process, so that the classic off-policy evaluation (OPE) estimators may fail to identify the true value function of the target policy. 
	In this paper, we study the statistical properties of OPE in confounded MDPs with observable instrumental variables. 
	Specifically, we propose a two-stage estimator based on the instrumental variables and establish its statistical properties in the confounded MDPs with a linear structure. 
	For non-asymptotic analysis, we prove a $\mathcal{O}(n^{-1/2})$-error bound where $n$ is the number of samples.
	For asymptotic analysis, we prove that the two-stage estimator is asymptotically normal with a typical rate of $n^{1/2}$. 
	To the best of our knowledge, we are the first to show such statistical results of the two-stage estimator for confounded linear MDPs via instrumental variables.
\end{abstract}


\section{Introduction}
\label{sec: introduction}
Offline reinforcement learning (offline RL, \cite{sutton2018reinforcement,levine2020offline}) is a machine learning paradigm which aims to learn a policy for sequential decision-making from a pre-collected offline dataset. 
With its huge empirical success \citep{mnih2015human,lillicrap2015continuous,fujimoto2019off,kidambi2020morel}, offline RL has also been studied extensively from a theoretical perspective in recent years \citep{chen2019information,jin2020pessimism,duan2020minimax,agarwal2020optimistic,duan2021risk,zhan2022offline}. 
A critical problem in offline RL is off-policy evaluation (OPE), which aims to estimate the long-term expected cumulative reward received by a target policy using the offline dataset collected by a different behaviour policy \citep{duan2020minimax,bennett2021off,min2021variance}. 

Typically, existing works on OPE develop algorithms and theories based on the model of Markov decision process (MDP), where the target policy is evaluated using data collected by a behavior policy.
The limitation is that this model cannot characterize the situation when unobserved confounders exist in offline data generation, as is often the case in many real-world applications.
For example, in data collection in the domains of healthcare, a physician may take treatments based on a patient’s mental state or socioeconomic status, which is hard or prohibited to be recorded in the data due to privacy concerns. 
Meanwhile, such information can affect the clinical outcomes, which makes the offline data confounded \citep{zhang2016markov,tennenholtz2020off}.
To better adapt to these applications, \citet{zhang2016markov,wang2021provably,liao2021instrumental, bennett2021off} propose and study confounded MDPs.
In a confounded MDP, there exist unobserved confounders which can influence both the action and the reward, causing confounding issues \citep{pearl2009causality,zhang2016markov} in the collected data.
As a result, conventional MDP-based OPE estimators, which do not handle the confounding issues in the data, may fail to identify the true value of the target policy in this case, causing an estimation bias.

In this work, we present the first statistical results of OPE in a confounded MDP based on the tool of instrumental variable (IV) \citep{angrist1995identification,brookhart2007preference,baiocchi2014instrumental,michael2020instrumental}. 
IV is a widely-used tool in statistics, econometrics, and causal inference, which can help us to identify the desired causal effects in the face of unobserved confounders.
For example, in healthcare domains, existing works have explored various kinds of IVs across preference-based IV \citep{brookhart2007preference,komorowski2018artificial} and differential-travel-time-based IV \citep{lorch2012differential,michael2020instrumental,chen2021estimating}.
Generally, we can identify the true value of a target policy in a confounded MDP using only observable variables when a set of IVs is available.
Recently, \citet{li2021causal,liao2021instrumental} have paid attention to applying IVs in addressing confounding issues in RL problems 
However, few works study the statistical properties of doing OPE in confounded MDPs. 
In particular, it remains open that 
\textbf{(a)} how can we design a both statistically and computationally efficient OPE estimator based on observable IVs for confounded MDPs? 
\textbf{(b)} how many samples are sufficient to guarantee accurate estimation with such an estimator?
\textbf{(c)} can we perform statistical inference using this estimator?
All these critical questions necessitate further theoretical understandings of OPE in confounded MDPs via IVs from a statistical perspective.


\subsection{Contribution}

In this paper, we give affirmative answers to the above questions in our study of OPE in confounded MDPs with IVs.
Specifically, our contributions are three-fold, which we present in the following.

\vspace{3mm}
\noindent
\textbf{(a) A Two-stage estimator.} 
We propose and study a \emph{two-stage estimator} for OPE in an infinite-horizon confounded linear MDP based on instrumental variables.  
This estimator borrows the idea from semiparametric regression \citep{yao2010efficient, darolles2011nonparametric} which handles the problem of endogeneity \citep{wooldridge2015introductory} in the data.
We extend such an idea from semiparametric statistics to the RL paradigm and affirmatively answer the Question \textbf{(a)}.
We start from a simpler confounded multi-armed bandit case in Section \ref{sec: bandit} to illustrate the core method.
Then we go further to study general confounded MDPs in Section \ref{sec: mdp}.

\vspace{3mm}
\noindent
\textbf{(b) Non-asymptotic analysis.} 
We present a non-asymptotic analysis for the two-stage estimator, proving that the statistical error is $\widetilde{\mathcal{O}}\left(\sqrt{\frac{d}{(1-\gamma)^4n}}\right)$, where 
$\widetilde{\mathcal{O}}$ hides logarithm factors, universal constants, and higher order terms.
Here $d$ is the dimension of the feature of the confounded MDP, $\gamma$ is the discount factor, and $n$ is the number of samples.
This analysis then answers Question \textbf{(b)}.

\vspace{3mm}
\noindent
\textbf{(c) Asymptotic analysis.} 
Furthermore, we present an asymptotic analysis for the two-stage estimator. 
We prove that the estimator is asymptotically normal with a typical rate of $n^{1/2}$. 
This answers Question \textbf{(c)} in the limit sense, by which we can use the estimator for statistical inference.


\subsection{Related Work}

\paragraph{Off-policy evaluation in MDPs.}
Our work is closely related to a huge body of works on off-policy evaluation (OPE) in MDPs \citep{hirano2003efficient,mannor2004bias,jong2007model,grunewalder2012modelling,bertsekas1995neuro,dann2014policy, duan2020minimax,kallus2020double,min2021variance}. 
Typically, an OPE method falls into three paradigms: direct methods (DM) \citep{mannor2004bias}, importance sampling (IS) \citep{hirano2003efficient}, and doubly robust (DR) \citep{kallus2020double}. 
Most related to our work is DM, which directly estimates the unknowns of the underlying MDP from data \citep{mannor2004bias,jong2007model,grunewalder2012modelling}.
Moreover, in order to handle large state and action spaces, the technique of function approximation is then studied \citep{bertsekas1995neuro,dann2014policy,duan2020minimax,min2021variance}. Among them, \citet{duan2020minimax,min2021variance} also consider statistical properties of OPE in MDPs with linear features. 
Our work extends this line of research by studying the statistical properties of OPE in confounded MDPs with a linear feature where the offline data involves confounding issues.

\paragraph{Causal reinforcement learning.}
Due to potential confounding issues in real-world applications, causal RL has attracted great research interest recently \citep{zhang2016markov,lu2018deconfounding,bennett2021off,cui2021semiparametric,qiu2021optimal,li2021causal,liao2021instrumental,chen2021instrumental,lu2022pessimism}.
\citet{zhang2016markov} firstly proposed confounded MDPs, where unobserved confounders have implicit impacts on the data generation mechanism of action and reward. 
In this case, traditional algorithms may only find sub-optimal solutions. Then, \citet{lu2018deconfounding} extended the Actor-Critic method to identify the latent confounders. Later, \citet{bennett2021off} studied the OPE problem in the confounded MDPs and proved the value function can be identified with only one latent variable assumption. In the meantime, \citet{li2021causal, liao2021instrumental} applied the instrumental variable approach to tackle the model parameter estimation problem in continuous state and action spaces, which are the most related works. 
In particular, 
\citet{liao2021instrumental} studied a primal-dual formulation for solving offline confounded MDPs via IV from an optimization perspective. 
However, the statistical properties for addressing confounded MDPs with IV remain less studied. 
\citet{li2021causal} considered an online RL setting with confounding bias, which contrasts with our study in the offline setting.
Also, their results deal with continuous state and action spaces, and can not cover the discrete state and action spaces as we consider in our work. 
Besides, \citet{qiu2021optimal} and \citet{cui2021semiparametric} applied the instrumental variable techniques to studying the problem of optimal treatment regimes, which is a single-stage decision-making problem. This contrasts with our work on the problem of OPE in confounded MDPs which involves multi-stage decision-making. \citet{qiu2021optimal}
provided asymptotically nonparametric analysis for their algorithm, while the statistical properties of the algorithm proposed by \citet{cui2021semiparametric} remained unstudied.
Finally, from a practical perspective, \citet{chen2021instrumental} studied OPE aided with IV in deep RL settings.

\paragraph{Semiparametric reinforcement learning.}
Our approach to addressing confounded MDPs with observable instrumental variables is also related to 
the literature of semiparametric statistics and semiparametric RL \citep{van2000asymptotic, tsiatis2006semiparametric,ueno2008semiparametric,krishnamurthy2018semiparametric,kallus2020double}.
In semiparametric statistics, we always assume that the statistical model has infinite-dimensional parameters \citep{van2000asymptotic, tsiatis2006semiparametric}. 
For semiparametric RL, \citet{ueno2008semiparametric} firstly studied the problem of policy evaluation from the semiparametric statistical viewpoint. 
Recently, \citet{krishnamurthy2018semiparametric} proposed a semiparametric contextual bandit, where they assumed an unobservable confounding term exists in the statistical model and showed a sublinear regret bound. 
Then, \citet{kallus2020double} studied the semiparametric efficiency of off-policy evaluation (OPE) estimators and found existing estimators might be inefficient in the semiparametric MDPs' framework.

\section{Preliminaries}
\label{sec: confounded MDP}
We study the problem of \emph{off-policy evaluation} (OPE) in a confounded MDP via instrumental varaibles.
We define the notation we use in this work in Section \ref{subsec: notation}. 
Then we introduce the confounded MDP in Section \ref{subsec: cmdp}.
After, we formulate the OPE task in confounded MDPs and point out the problem of endogeneity in Section \ref{subsec: data and eval}. 
Finally, we present the tool of instrumental variable (IV) from causal inference in Section \ref{subsec: instrumental variable}.

\subsection{Notation}\label{subsec: notation}

We use $\Delta(\XM)$ to denote the set of probability distributions on a set $\XM$. 
We use superscript $(t)$ for time index and subscript $i$ for sample index. 
We use boldface letters to denote vectors and matrices.
For any vector $\boldsymbol{a}$, we use ${a}^{(k)}$ to denote its $k$-th element and use $\|\boldsymbol{a}\|_2$ to denote its $\ell_2$-norm.
For any matrix $\boldsymbol{A}$, we use $\|\boldsymbol{A}\|$ to denote its operator $\ell_2$-norm and use $\sigma_{\min}(\boldsymbol{A})$, $\sigma_{\max}(\boldsymbol{A})$ to denote its minimal and maximal eigenvalue respectively.
For any vector $\boldsymbol{a}$ and semi-positive definite matrix $\boldsymbol{A}$, we define $\|\boldsymbol{a}\|_{\boldsymbol{A}} = (\boldsymbol{a}^\top\boldsymbol{A}\boldsymbol{a})^{1/2}$.
For any random variable sequence $X_1, \ldots, X_n, \ldots$,  $X_n=o_P(1)$  means the sequence converges to zero in probability, and $X_n=O_P(1)$ means the sequence is bounded in probability.

\subsection{Confounded Markov Decision Processes}\label{subsec: cmdp}
We consider an infinite-horizon confounded Markov decision process (confounded MDP) \citep{bennett2021off}, which is represented by a tuple $(\SM,\AM,\EM,P,P_{\EM},R,\gamma)$.
Here $\SM$ is the state space,  $\AM$ is the action space, and $\EM$ is the confounder space.
We assume that both $\SM$ and $\AM$ are finite spaces with cardinalities $S=|\SM|$ and $A=|\AM|$, respectively. 
Also, we assume that the confounder space $\EM\subseteq[-1,1]\subseteq\mathbb{R}$.
The mapping $P\colon \SM\times\AM\mapsto\Delta(\SM)$ is the transition kernel, which gives the distribution of the next state $s^{(t+1)}$ given the current state $s^{(t)}$ and action $a^{(t)}$. 
The mapping $P_{\EM}\colon\SM\times\AM\mapsto\Delta(\EM)$ gives the distribution of the confounder $\epsilon^{(t)}$.
The mapping $R\colon \SM\times\AM\mapsto\mathbb{R}$ is the deterministic reward function and the reward is given by $r^{(t)}=R(s^{(t)},a^{(t)})+\epsilon^{(t)}$. 
Finally, the constant $\gamma\in[0,1)$ denotes the discount factor, with $\gamma=0$ corresponding to a confounded bandit.

In addition, we assume that both the transition kernel $P$ and the reward function $R$ of the confounded MDP satisfy a linear structure assumption. 
Such an assumption is widely used in the literature of RL with linear function approximations
\citep{jin2020provably, duan2020minimax, wang2020statistical, min2021variance}.

\begin{asmp}[Linear structure]\label{assump: linear}
We assume that there exist a known feature mapping $\boldsymbol{\phi}\colon \SM\times\AM\mapsto\mathbb{R}^d$,
an unknown vector-valued function $\boldsymbol{\nu}\colon \SM\mapsto\mathbb{R}^d$, and an unknown vector $\boldsymbol{w} \in\mathbb{R}^d$ such that
\begin{align*}
    P(s^{\prime}|s,a)=\boldsymbol{\phi}(s,a)^\top\boldsymbol{\nu}(s^{\prime}),\quad R(s,a)=\boldsymbol{\phi}(s,a)^\top \boldsymbol{w},
\end{align*}
for all $(s,a,s^{\prime})\in\SM\times\AM\times\SM$.
We further assume that $\boldsymbol{\phi}$, $\boldsymbol{\nu}$, and $\boldsymbol{w}$ are normalized in the sense that 
\begin{align*}
\sup_{(s,a)\in\SM\times\AM}\|\boldsymbol{\phi}(s,a)\|_2\leq 1, \quad \|\boldsymbol{w}\|_2\leq 1, \quad \mbox{ and } \quad 
\sup_{a\in\AM}\bigg\|\sum_{s\in\SM}\boldsymbol{\nu}(s)\boldsymbol{\phi}(s,a)^\top\bigg\|\leq 1.
\end{align*}
\end{asmp}

\begin{rem}
    For finite state and action spaces $\mathcal{S}$ and $\mathcal{A}$, the transition kernel $P$ and the reward function $R$ allow for a canonical linear representation given by $d = SA$, $\boldsymbol{\phi}(s,a) = \mathbf{e}_{(s,a)}$, $\boldsymbol{\nu}(s^{\prime}) = (P(s^{\prime}|s,a))_{(s,a)\in\mathcal{S}\times\mathcal{A}}$, and $\boldsymbol{w} = (R(s,a))_{(s,a)\in\mathcal{S}\times\mathcal{A}}$.
    However, the linear assumption allows us to consider confounded MDPs with prohibitive large state spaces while the transition and reward admits low rank decompositions, that is, $d \ll SA$.
\end{rem}

\subsection{Data Generation and Policy Evaluation.}\label{subsec: data and eval}

To study OPE in a confounded MDP, we assume the access to a \emph{confounded} offline dataset generated from the confounded MDP.
We require that the confounder is correlated with the state-action pair and the confounder itself is not recorded by the dataset.
Mathematically, we require that $P_{\mathcal{E}}$ is not degenerated when considered as a function on $\mathcal{S}\times\mathcal{A}$.
This causes the so called problem of endogeneity \citep{wooldridge2015introductory} in the linear equation $r_i=R(s_i,a_i)+\epsilon_i$ because $(s_i,a_i)$ (explanatory variable) and $\epsilon_i$ (error term) are correlated. 

For example, in the healthcare, the patient's socioeconomic status serves as the confounder.
Physicians tend to use more expensive treatments (action) for wealthier patients.  
Meanwhile, wealthier patients tend to have better treatment outcomes (reward) \citep{tennenholtz2020off}.
Due to privacy regulations, however, such confounder information can not be covered by the electronic health records (offline data).

The goal of OPE is to evaluate a known target policy $\pi^e:\mathcal{S}\mapsto\Delta(\mathcal{A})$ under the assumption that the $\epsilon^{(t)}$ is independent of $(s^{(t)}, a^{(t)})$ and is zero-mean. 
This means that $P_{\mathcal{E}}(\cdot|s,a) = P_{\mathcal{E}}(\cdot)$ for some $P_{\mathcal{E}}\in\Delta(\mathcal{E})$ with zero-mean.
Specifically, we need to estimate the value function of $\pi^e$ in the confounded MDP, defined as the expected cumulative reward given the initial state,
\begin{align}\label{eq: V pi e}
    V_{\pi^e}(s)=\mathbb{E}_{\pi^{e}}\left[\sum_{t=1}^{+\infty}\gamma^{t-1}r^{(t)}\middle| s^{(1)}=s\right],\quad \forall s\in\mathcal{S},
\end{align}
where the expectation is taken with respect to the trajectories induced by $\pi^e$, i.e., $s^{(1)}=s$, $a^{(t)}\sim \pi^e(\cdot|s^{(t)})$, $\epsilon^{(t)}\sim P_{\EM}(\cdot)$, $r^{(t)}=R(s^{(t)}$,$a^{(t)})+\epsilon^{(t)}$, and $s^{(t+1)}\sim P(\cdot|s^{(t)},a^{(t)})$.
Note that in the offline data the confounder influences both the state-action and the reward\footnote{
    Here we only consider the reward is confounded for simplicity and readability. In fact, our algorithm and analysis can be readily extended to the model where the state transition is also confounded.
}. 
But the confounder is unobservable for the learner. 
This causes a confounding issue \citep{pearl2009causality} which prohibits standard OPE methods \citep{duan2020minimax}.

\subsection{Instrumental Variables.}\label{subsec: instrumental variable}
We study the method of instrumental variable (IV) \citep{pearl2009causality} for overcoming the problem of endogeneity and identifying the value function \eqref{eq: V pi e} of target policy $\pi^e$ from the confounded observational data.
Intuitively, we assume that in the offline data, there is an instrumental variable that influences the reward through state and action, while conditioning on which the confounder has zero mean.
Formally, we assume the following.

\begin{asmp}[Instrumental variables]\label{assump: instrumental variable}
    We assume that in the offline dataset there is an instrumental variable $z_i\in\ZM$ taking value in some discrete space $\ZM$ with $Z=|\ZM|$ which satisfies that
    \begin{enumerate}
        \item The instrumental variable influences the reward only through state and action, i.e., $z_i\indep\epsilon_i \mid (s_i,a_i)$;
        \item Conditioning on the instrumental variable the confounder has zero mean, i.e., $\mathbb{E}[\epsilon_i|z_i]=0$. 
    \end{enumerate}
\end{asmp}

In the related literatures there exist different types of definitions of a instrumental variable \citep{chen2021instrumental, chen2021estimating, liao2021instrumental,li2021causal, li2021self}.
In this work we focus on the instrumental variable satisfying Assumption \ref{assump: instrumental variable}, which is similar to that considered by \cite{li2021causal}.

Under Assumption \ref{assump: instrumental variable}, we consider the offline dataset $\mathbb{D}$ denoted by the following $n$ i.i.d. samples,
\begin{align}
    \mathbb{D}=\big\{\big(s_i,a_i,z_i,r_i,s_i^{\prime}\big)\big\}_{i=1}^n,
\end{align}
where each sample $(s_i,a_i,z_i,r_i,s_i^{\prime})$ is generated independently according to 
$z_i\sim p(z_i)$, $s_i,a_i\sim \rho(\cdot,\cdot|z_i)$, $\epsilon_i \sim P_{\mathcal{E}}(\cdot|s_i,a_i)$, $r_i = R(s_i,a_i) + \epsilon_i$, and $s_i^{\prime}\sim P(\cdot|s_i,a_i)$.
Here $\rho(s,a|z)p(z)$ defines a joint distribution on $\mathcal{Z}\times\mathcal{S}\times\mathcal{A}$.
We use $\mathbb{P}$ and $\mathbb{E}$ to denote the probability and expectation with respect to the randomness of $\mathbb{D}$.
In the sequel, we give two concrete examples of confounded MDPs when Assumption \ref{assump: instrumental variable} holds.

\begin{example}[Reduction to MDP \citep{chen2019information}]\label{example: mdp}
For a standard MDP, we set $z_i=(s_i,a_i)$ and we have that $\mathbb{E}[\epsilon_i|s_i,a_i]=0$. 
The offline dataset reduces to a standard offline RL dataset.
\end{example}

\begin{example}[Decomposible actions, motivated by \cite{li2021causal}]\label{example: decomposible action}
    Consider that the action $a_i$ can be decomposed into two parts: $a_i=a_{i,1}+a_{i,2}$.
    The transition and the reward are given by $s_{i}^{\prime} = c_0+c_1s_i+c_2a_i+\zeta_i$, $r_i = R(s_i,a_i)+\epsilon_i$ where $R(s_i,a_i)=r_0+r_1a_i+r_2s_ia_i+r_3a_i^2$ and $\epsilon_t = b a_{i,2} + \eta_i$. Here $\zeta_i$ and $\eta_i$ are both independent zero-mean random variables. Also, the action $a_{i,2}$ is taken as a zero-mean random variable. In this case, by choosing $z_i = (s_i,a_{i,1})$, one can check that the condition of IV (Assumption \ref{assump: instrumental variable}) is satisfied.
\end{example}

Finally, to ensure an efficient estimation of the policy value function, we also make the following coverage assumption on the offline dataset distribution $\mathbb{P}$ which depends on $p$, $\rho$, $P_{\mathcal{E}}$, and $P$.
Under the existence of instrumental variables, we adapt the standard coverage assumptions in offline RL to the following assumption.
 
\begin{asmp}[Instrumental sufficient coverage]\label{assump: coverage} 
    Denote the conditional feature of $(s_i,a_i)$ given $z_i$ as 
    \begin{align}\label{eq: phi rho}
        \boldsymbol{\phi}_{\rho}(z):=\sum_{(s,a)\in\SM\times\AM}\boldsymbol{\phi}(s,a)\rho(s,a|z) = \mathbb{E}[\boldsymbol{\phi}(s_i,a_i)|z_i=z]\in\mathbb{R}^d.
    \end{align}
    Also, we denote the covariance matrix of the conditional feature $\boldsymbol{\phi}_{\rho}(z_i)$ as 
    \begin{align}\label{eq: Sigma}
        \boldsymbol{\Sigma}:=\sum_{z\in\mathcal{Z}}p(z)\boldsymbol{\phi}_{\rho}(z)\boldsymbol{\phi}_{\rho}(z)^\top=\mathbb{E}\Big[\boldsymbol{\phi}_{\rho}(z_i)\boldsymbol{\phi}_{\rho}(z_i)^\top\Big]\in\mathbb{R}^{d\times d}.
    \end{align}
    Then we assume that the covariance matrix $\boldsymbol{\Sigma}$ is not singular, that is,
    \begin{align}
        \underline{\sigma}=\sigma_{\min}(\boldsymbol{\Sigma})>0.
    \end{align}
\end{asmp}
Notably, for a standard MDP and $z_i = (s_i,a_i)$, i.e., in Example \ref{example: mdp}, Assumption \ref{assump: coverage} also reduces to the standard distribution coverage assumption for offline RL in a linear MDP \citep{duan2020minimax, wang2020statistical, min2021variance}.
As we can see in the following, the conditional feature $\boldsymbol{\phi}_{\rho}$ plays an important rule in the identification of the value function.
This motivates us to use the covariance of $\boldsymbol{\phi}_{\rho}$ to characterize the coverage property of the offline dataset $\mathbb{D}$.

In the coming Section \ref{sec: bandit} and \ref{sec: mdp}, we explain in general how the existence of instrumental variables satisfying Assumption \ref{assump: instrumental variable} can help identify the value function \eqref{eq: V pi e}.
Based on the identification, we propose an efficient \emph{two-stage estimator} to estimate the value function using the offline dataset $\mathbb{D}$.
To illustrate the core method, we start from a simpler confounded multi-armed bandit case ($\gamma=0$ and $S=1$) in Section \ref{sec: bandit}.

\section{Warm-up: Confounded Multi-armed Bandits}
\label{sec: bandit}
As a warm-up case, let us first consider $\gamma=0$ and $S=1$, which corresponds to a confounded multi-armed bandit \citep{xu2021deep,qin2022adaptivity}.
In this case, the policy $\pi^e$ is a distribution on the action space $\mathcal{A}$ which we denote as $\pi^e(\cdot)$, and the feature mapping $\boldsymbol{\phi}$ is a mapping of $\mathcal{A}\mapsto\mathbb{R}^d$.
Also, the value function of $\pi^e$ reduces to a scalar $V_{\pi^e}$ which can be written as
\begin{align}\label{eq: V bandit}
    V_{\pi^e}=\mathbb{E}_{\pi^e}\Big[r^{(1)}\Big]=\mathbb{E}_{\pi^e}\Big[\boldsymbol{\phi}(a^{(1)})^\top \boldsymbol{w} + \epsilon^{(1)}\Big]=\sum_{a\in\AM}\pi^e(a)\boldsymbol{\phi}(a)^\top \boldsymbol{w},
\end{align}
where the second equality holds due to the assumption in Section \ref{subsec: data and eval} that, during policy evaluation, $\epsilon^{(1)}$ is independent of $(s^{(1)}, a^{(1)})$ and is zero-mean.

Meanwhile, since the time horizon is $1$ under and the state space $\SM$ is a singleton under $\gamma=0$ and $S=1$, we can simplify the dataset $\mathbb{D}$ to $\mathbb{D}_{\texttt{bandit}} = \{(a_i,z_i,r_i)\}_{i=1}^n$.
In view of \eqref{eq: V bandit}, we propose to estimate the value function $V_{\pi^e}$ via estimating the parameter $\boldsymbol{w}$ using the offline dataset $\mathbb{D}_{\texttt{bandit}}$.

\subsection{A Two-stage Estimator}\label{subsec: bandit two stage estimator}
Now we introduce a two-stage estimator $\widehat{\boldsymbol{w}}$ of the unknown vector $\boldsymbol{w}$.
Recall that in the confounded offline dataset $\mathbb{D}_{\texttt{bandit}}$, $a_i$, $\epsilon_i$, and $r_i$ satisfy that $r_i=\boldsymbol{\phi}(a_i)^\top \boldsymbol{w} + \epsilon_i$.
However, since the action $a_i$ is correlated with confounder $\epsilon_i$, we cannot directly apply the ordinary least square regression to estimate $\boldsymbol{w}$.
To this end, we make use of the instrumental variable $z_i$. Taking conditional expectation given $z_i$, we can obtain that 
\begin{align}\label{eq: deduction bandit}
    \mathbb{E}[r_i|z_i]=\mathbb{E}\Big[\boldsymbol{\phi}(a_i)^\top\Big|z_i\Big] \boldsymbol{w} + \underbrace{\mathbb{E}[\epsilon_i|z_i]}_{ = 0 }=\boldsymbol{\phi}_{\rho}(z_i)^\top \boldsymbol{w},
\end{align}
where the second equality holds due to $z_i$ is a instrumental variable, i.e., Assumption \ref{assump: instrumental variable}, and the definition of $\boldsymbol{\phi}_{\rho}(z)$ in \eqref{eq: phi rho}.
Now we multiply another $\boldsymbol{\phi}_\rho(z_i)$ on both sides of \eqref{eq: deduction bandit} and take expectation with respect to the instrumental variable $z_i$, which gives that
\begin{align*}
    \mathbb{E}[\boldsymbol{\phi}_{\rho}(z_i)r_i]=\mathbb{E}[\boldsymbol{\phi}_{\rho}(z_i)\mathbb{E}[r_i|z_i]]=\mathbb{E}\Big[\boldsymbol{\phi}_{\rho}(z_i)\boldsymbol{\phi}_{\rho}(z_i)^\top\Big] \boldsymbol{w} = \boldsymbol{\Sigma} \boldsymbol{w},
\end{align*}
where the covariance matrix $\boldsymbol{\Sigma}$ is defined in \eqref{eq: Sigma}. 
Under Assumption \ref{assump: coverage}, we can then solve $\boldsymbol{w}$ as 
\begin{align}\label{eq: alpha}
    \boldsymbol{w} = \boldsymbol{\Sigma}^{-1}\mathbb{E}[\boldsymbol{\phi}_{\rho}(z_i)r_i]:=\boldsymbol{\Sigma}^{-1}\boldsymbol{\tau}.
\end{align}
This motivates us to propose a two-stage estimator for the unknown vector $\boldsymbol{w}$.
It borrows from the idea of semiparametric regression \citep{yao2010efficient, darolles2011nonparametric} and generalizes the well-known two-stage least square regression \citep{angrist1995two}.
According to \eqref{eq: alpha}, we first estimate the conditional expectation $\boldsymbol{\phi}_{\rho}(z_i)$ in a non-parametric manner.
Then by regressing $r_i$ against the estimated feature mapping, which we denote as $\boldsymbol{\phi}_{\widehat{\rho}}(z_i)$, we can estimate the unknown vector $\boldsymbol{w}$.

\vspace{3mm}
\noindent
\textbf{Stage one.}
Specifically, given offline dataset $\mathbb{D}_{\texttt{bandit}}$, we obtain the estimate of the conditional probability $\rho(a|z)$ as an empirical average on the data, i.e., we estimate $\rho(a|z)=\mathbb{P}(a_i=a|z_i=z)$ as 
\begin{align}\label{eq: hat rho}
    \widehat{\rho}(a|z) = \frac{\sum_{i=1}^n\mathbbm{1}\{a_i=a,z_i=z\}}{\max\{\sum_{i=1}^n\mathbbm{1}\{z_i=z\},1\}},\quad \forall (a,z)\in\mathcal{A}\times\mathcal{Z}.
\end{align}
Then we plug $\widehat{\rho}$ into $\boldsymbol{\phi}_{\rho}$ and obtain to the estimate of the conditional expectation $\boldsymbol{\phi}_{\rho}(z)$ as 
\begin{align}\label{eq: phi hat rho}
    \boldsymbol{\phi}_{\widehat{\rho}}(z) = \sum_{a\in\AM}\boldsymbol{\phi}(a)\widehat{\rho}(a|z),\quad \forall z\in\mathcal{Z}.
\end{align}

\noindent
\textbf{Stage two.}
After obtaining the estimate of the conditional feature $\boldsymbol{\phi}_{\widehat{\rho}}(z)$, we apply the ordinary least square regression for $r_i$ on the estimated feature mapping $\boldsymbol{\phi}_{\widehat{\rho}}$ and obtain the two-stage estimator of $\boldsymbol{w}$ as
\begin{align}\label{eq: hat alpha}
    \widehat{\boldsymbol{w}}= \Bigg(\frac{1}{n}\sum_{i=1}^n\boldsymbol{\phi}_{\widehat{\rho}}(z_i)\boldsymbol{\phi}_{\widehat{\rho}}(z_i)^\top\Bigg)^{-1}\Bigg(\frac{1}{n}\sum_{i=1}^n\boldsymbol{\phi}_{\widehat{\rho}}(z_i)r_i\Bigg):=\widehat{\boldsymbol{\Sigma}}^{-1}\widehat{\boldsymbol{\tau}}.
\end{align}
We remark that for $n$ large enough the smallest eigenvalue of $\widehat{\boldsymbol{\Sigma}}$ approximates that of $\boldsymbol{\Sigma}$, which guarantees the invertibility of $\widehat{\boldsymbol{\Sigma}}$. 
In practise, one can consider a ridge-regression-style variant of $\widehat{\boldsymbol{w}}$, where a regularization term of $\lambda \boldsymbol{I}$ is involved.
For simplicity we only focus on the estimator given by \eqref{eq: hat alpha} in this work.
Finally, by plugging this estimator \eqref{eq: hat alpha} into \eqref{eq: V bandit}, we can derive our estimate of the value function of policy $\pi^e$ as 
\begin{align}\label{eq: hat V bandit}
    \widehat{V}_{\pi^e}=\sum_{a\in\AM}\pi^e(a)\boldsymbol{\phi}(a)^\top \widehat{\boldsymbol{w}},\quad\text{$\widehat{\boldsymbol{w}}$ is defined in \eqref{eq: hat alpha}.}
\end{align}

Note that under Example \ref{example: mdp}, the two-stage estimator \eqref{eq: hat alpha} reduces to an ordinary least square estimator of $\boldsymbol{w}$, which coincides with the classical parameter estimation for linear bandits \citep{abbasi2011improved,duan2020minimax}.
Our discussions on the confounded multi-armed bandit case can be directlty extended to the confounded contextual bandit setting where $S>1$.
Also, the result for confounded bandits plays an important role in our investigation of general confounded MDPs in Section \ref{sec: mdp}.

\subsection{Theoretical Properties of the Two-stage Estimator} 
In this section, we give the statistical properties of the two-stage value function estimator we proposed in Section \ref{subsec: bandit two stage estimator}.
In particular, we establish the non-asymptotic convergence rate (see Corollary \ref{cor: non asymptotic bandit}) and the asymptotic distribution (see Corollary \ref{cor: asymptotic bandit}) of the estimator \eqref{eq: hat V bandit} respectively.
In Section \ref{sec: mdp} we establish the corresponding results for general confounded MDPs and the results in the following are direct corollaries.

For a non-asymptotic analysis of the value function estimator \eqref{eq: hat V bandit}, we have the following result.
\begin{cor}[Non-asymptotic Analysis: Confounded Bandit]\label{cor: non asymptotic bandit}
Under Assumptions \ref{assump: linear}, \ref{assump: instrumental variable}, and \ref{assump: coverage}, when $\gamma=0$ and $S=1$, the value function estimator \eqref{eq: hat V bandit} of $\pi^e$ satisfies that, with probability at least $1-\delta$,
\begin{align}\label{eq: non asymptotic bandit}
    \big|\widehat{V}_{\pi^e} - V_{\pi^e}\big|\leq \widetilde{\mathcal{O}}\Bigg(  \sqrt{\frac{\|\mathbb{E}_{a\sim\pi^e}[\boldsymbol{\phi}(a)]\|_{\boldsymbol{\Sigma}^{-1}}^2d}{\underline{\sigma} n}} \,\Bigg),
\end{align}
where $\widetilde{\mathcal{O}}(\cdot)$ hides universal constants, logarithm factors, and higher-order terms.
\end{cor}
\begin{proof}[Proof of Corollary \ref{cor: non asymptotic bandit}]
    See Appendix \ref{subsec: proof non asymptotic bandit} for a detailed proof.
\end{proof}
We note that the non-asymptotic upper bound \eqref{eq: non asymptotic bandit} scales with the term $\|\mathbb{E}_{a\sim\pi^e}[\boldsymbol{\phi}(a)]\|_{\boldsymbol{\Sigma}^{-1}}$. 
This term characterizes how the offline data $\mathbb{D}_{\texttt{bandit}}$ covers the feature $\mathbb{E}_{a\sim\pi^e}[\boldsymbol{\phi}(a)]$.
Such a terms is directly bounded by a factor of $\underline{\sigma}^{-1/2}$ due to the following argument,
\begin{align}\label{eq: cover bandit bound}
    \|\mathbb{E}_{a\sim\pi^e}[\boldsymbol{\phi}(a)]\|_{\boldsymbol{\Sigma}^{-1}} = \sqrt{\mathbb{E}_{a\sim\pi^e}[\boldsymbol{\phi}(a)]^\top\boldsymbol{\Sigma}^{-1}\mathbb{E}_{a\sim\pi^e}[\boldsymbol{\phi}(a)]}\leq \underline{\sigma}^{-\frac{1}{2}}\cdot\|\mathbb{E}_{a\sim\pi^e}[\boldsymbol{\phi}(a)]\|_2\leq \underline{\sigma}^{-\frac{1}{2}},
\end{align}
where the last inequality follows from the normalization assumption in Assumption \ref{assump: linear}.
In conclusion, the two-stage value function estimator enjoys a $\widetilde{\mathcal{O}}(\sqrt{d/\underline{\sigma}^2n})$ statistical rate.
To ensure an $\epsilon$-accuracy, the number of samples needed is approximately $n=\widetilde{\mathcal{O}}(d/\underline{\sigma}^2\epsilon^2)$.
See Section \ref{subsec: non asymptotic mdp} for further technical discussions.

For asymptotic analysis of the value function estimator \eqref{eq: hat V bandit}, we have the following result.

\begin{cor}[Asymptotic Analysis: Confounded Bandit]\label{cor: asymptotic bandit}
    Under Assumptions \ref{assump: linear}, \ref{assump: instrumental variable}, and \ref{assump: coverage}, when $\gamma=0$ and $|\SM|=1$, it holds that the value function estimator \eqref{eq: hat V bandit} of $\pi^e$ satisfies that, asymptotically,
    \begin{align}
      \sqrt{n}\cdot\big(\widehat{V}_{\pi^e}-V_{\pi^e}\big)\,\stackrel{d}{\rightarrow}\,\,\mathcal{N}\big(\boldsymbol{0},\boldsymbol{\phi}_{\pi^e}^\top\boldsymbol{\Sigma}^{-1}\boldsymbol{\Phi}\boldsymbol{\Sigma}^{-1}\boldsymbol{\phi}_{\pi^e}\big),
    \end{align}
    where the vector $\boldsymbol{\phi}_{\pi^e}\coloneqq\sum_{a\in\AM}\boldsymbol{\phi}(a)\pi^e(a)\in\mathbb{R}^d$ and the matrix  $\boldsymbol{\Phi}\in\mathbb{R}^{d\times d}$ is defined by
    $$\boldsymbol{\Phi}^{(k,l)}\coloneqq\mathbb{E}[\sigma^2(z_i)(\boldsymbol{\phi}_\rho(z_i))_k(\boldsymbol{\phi}_{\rho}(z_i))_l],$$ for $1\leq k,l\leq d$ with $\sigma^2(z_i)=\mathbb{E}[\epsilon^2_i|z_i]$.
\end{cor}
This result establishes the asymptotic normality of the two-stage estimator \eqref{eq: hat V bandit} and allows for statistical inference in the limit sense.
The above two results are direct corollaries of Theorem \ref{thm: non asymptotic mdp} and Theorem \ref{thm: asymptotic mdp} in Section \ref{sec: mdp}, where we develop our investigation on the two-stage estimator for general confounded MDPs.

\section{Confounded Markov Decision Processes}
\label{sec: mdp}
In this section, we study OPE in a general confounded MDP with instrumental variables, i.e., $\gamma\in(0,1)$ and $S>1$.
We propose and analyze a \emph{two-stage estimator} of the value function $V_{\pi^e}$, based on our discussion of confounded bandits in Section \ref{sec: bandit}.
To derive such a two-stage estimator, we first identify $V_{\pi^e}$ via observational data using the tool of instrumental variable.
Before we state the result, to simplify the notation, we define the feature mapping $\boldsymbol{\phi}_{\pi^e}:\mathcal{S}\mapsto\mathbb{R}^d$ as
\begin{align}\label{eq: phi pi e}
    \boldsymbol{\phi}_{\pi^e}(s) := \sum_{a\in\mathcal{A}}\boldsymbol{\phi}(s,a)\pi^e(a|s) = \mathbb{E}_{a\sim\pi^e(\cdot|s)}[\boldsymbol{\phi}(s,a)]\in\mathbb{R}^d.
\end{align}
Then we have the following proposition for identifying the value function.

\begin{prop}[Identification of Policy Value]\label{prop: identification policy value}
    For any policy $\pi^e:\mathcal{S}\mapsto\Delta(\mathcal{A})$, it holds that 
    \begin{align}\label{eq: identification policy value}
        V_{\pi^e}(s)=\boldsymbol{\phi}_{\pi^e}(s)^\top\boldsymbol{\theta},\quad \textnormal{where}\quad \boldsymbol{\theta}:=(\boldsymbol{I}-\gamma \boldsymbol{A})^{-1}\boldsymbol{w},
    \end{align}
    Here the vector $\boldsymbol{w}\in\mathbb{R}^d$ and the matrix $\boldsymbol{A}\in\mathbb{R}^{d\times d}$ are given by  
    \begin{align}
        \boldsymbol{w} = \mathbb{E}\Big[\boldsymbol{\phi}_{\rho}(z_i)\boldsymbol{\phi}_{\rho}(z_i)^\top\Big]^{-1}\mathbb{E}[\boldsymbol{\phi}_{\rho}(z_i)r_i]:=\boldsymbol{\Sigma}^{-1}\boldsymbol{\tau},
    \end{align}
    \begin{align}
        \boldsymbol{A} = \mathbb{E}\Big[\boldsymbol{\phi}_\rho(z_i)\boldsymbol{\phi}_\rho(z_i)^\top\Big]^{-1}\mathbb{E}\Big[\boldsymbol{\phi}_\rho(z_i)\boldsymbol{\phi}_{\pi^e}(s_i^{\prime})^\top\Big]:=\boldsymbol{\Sigma}^{-1}\boldsymbol{B}. 
    \end{align}
\end{prop}
\begin{proof}[Proof of Proposition \ref{prop: identification policy value}]
    See Appendix \ref{subsec: proof identification policy value} for a detailed proof.
\end{proof}
As we show in the proof of Proposition \ref{prop: identification policy value}, under the normalization assumption in Assumption \ref{assump: linear}, the matrix $\boldsymbol{I} - \gamma\boldsymbol{A}$ is invertible for any $\gamma\in(0,1)$.
Given the identification formula \eqref{eq: V mdp}, we propose to estimate the value function $V_{\pi^e}$ via estimating the parameter $\boldsymbol{\theta}$ using the offline dataset $\mathbb{D}$.

\subsection{A Two-stage Estimator for Confounded MDPs}
Now we introduce the two-stage estimator for estimating the parameter $\boldsymbol{\theta}$.
On the first stage, we estimate the conditional feature $\boldsymbol{\phi}_{\rho}$ in a non-parametric manner.
Specifically, we estimate the conditional probabiltiy $\rho(s,a|z) = \mathbb{P}(s_i=s, a_i=a|z_i = z)$ via an empirical average on the data, i.e., 
\begin{align}\label{eq: hat rho mdp}
    \widehat{\rho}(s,a|z)= \frac{\sum_{i=1}^n\mathbbm{1}\{s_i=s,a_i=a,z_i=z\}}{\max\{\sum_{i=1}^n\mathbbm{1}\{z_i=z\}, 1\}},\quad \forall (s,a,z)\in\mathcal{S}\times\mathcal{A}\times\mathcal{Z}.
\end{align}
Then we plug $\widehat{\rho}$ into $\boldsymbol{\phi}_{\rho}$ and obtain to the estimate of the conditional expectation $\boldsymbol{\phi}_{\rho}(z)$ as 
\begin{align}\label{eq: phi hat rho mdp}
    \boldsymbol{\phi}_{\widehat{\rho}}(z) = \sum_{(s,a)\in\SM\times\AM}\boldsymbol{\phi}(s,a)\widehat{\rho}(s,a|z),\quad \forall z\in\mathcal{Z}.
\end{align}
On the second stage, we estimate the vector $\boldsymbol{w}$ and the matrix $\boldsymbol{A}$ using the estimated $\boldsymbol{\phi}_{\widehat{\rho}}$ via
\begin{align}\label{eq: hat alpha A}
    \widehat{\boldsymbol{w}} = \widehat{\boldsymbol{\Sigma}}^{-1}\widehat{\boldsymbol{\tau}},\quad\text{and}\quad \widehat{\boldsymbol{A}} = \widehat{\boldsymbol{\Sigma}}^{-1}\widehat{\boldsymbol{B}},
\end{align}
where the vector $\widehat{\boldsymbol{\tau}}$ and the matices $\widehat{\boldsymbol{B}}$ and $\widehat{\boldsymbol{\Sigma}}$ are defined as
\begin{align}\label{eq: hat tau B Sigma}
    \widehat{\boldsymbol{\tau}} = \frac{1}{n}\sum_{i=1}^n\boldsymbol{\phi}_{\widehat{\rho}}(z_i)r_i,\quad \widehat{\boldsymbol{B}} = \frac{1}{n}\sum_{i=1}^n\boldsymbol{\phi}_{\widehat{\rho}}(z_i)\boldsymbol{\phi}_{\pi^e}(s_i^{\prime})^\top,\quad \widehat{\boldsymbol{\Sigma}} = \frac{1}{n}\sum_{i=1}^n\boldsymbol{\phi}_{\widehat{\rho}}(z_i)\boldsymbol{\phi}_{\widehat{\rho}}(z_i)^\top,
\end{align}
As is mentioned in Section \ref{sec: bandit}, in practise one can consider a ridge-regression-style variant of $\widehat{\boldsymbol{w}}$ and $\widehat{\boldsymbol{A}}$, where a regularization term of $\lambda \boldsymbol{I}$ is involved in $\widehat{\boldsymbol{\Sigma}}$.
For simplicity we still focus on the estimator given by \eqref{eq: hat alpha A} in this work.
Finally, by plugging this estimator \eqref{eq: hat alpha A} into \eqref{eq: identification policy value}, we can derive our estimate of the value function of policy $\pi^e$ as, for any $s\in\mathcal{S}$, 
\begin{align}\label{eq: hat V mdp}
    \widehat{V}_{\pi^e}(s)= \boldsymbol{\phi}_{\pi^e}(s)^\top\widehat{\boldsymbol{\theta}},\quad\text{where $\widehat{\boldsymbol{\theta}} = (\boldsymbol{I} - \gamma\widehat{\boldsymbol{A}})^{-1}\widehat{\boldsymbol{w}}$ and $\widehat{\boldsymbol{w}}$, $\widehat{\boldsymbol{A}}$ are defined in \eqref{eq: hat alpha A}.}
\end{align}
Under Example \ref{example: mdp} and $z_i = (s_i,a_i)$, our solution \eqref{eq: hat V mdp} coincides with the well-known solution of \emph{least square temporal difference Q-learning} (LSTDQ \citep{lagoudakis2003least}) for the standard MDPs.
Our two-stage value function estimator extends such a classic estimator to confounded MDPs with IV.

In the coming two Sections, we establish the statistical properties of the two-stage estimator of the value function.
We give non-asymptotic analysis to the estimator \eqref{eq: hat V mdp} in Section \ref{subsec: non asymptotic mdp}, and we give the asymptotic distribution of \eqref{eq: hat V mdp} in Section \ref{sec: mdp}.
Detailed proofs of our theoretical results are in Appendix \ref{sec: proof mdp}.

\subsection{Non-asymptotic Analysis}\label{subsec: non asymptotic mdp}
In this section, we establish the non-asymptotic convergence rate of the two-stage estimator \eqref{eq: hat V mdp}. 
To state our result, we first define the visitation measure $d_{\pi^e,\underline{s}}\in\Delta(\mathcal{S}\times\mathcal{A})$ of $\pi^e$ as the cumulative probability that the agent visits $(s,a)\in\mathcal{S}\times\mathcal{A}$ when starting from $\underline{s}\in\mathcal{S}$. That is,
\begin{align*}
    d_{\pi^e,\underline{s}}(s,a) = (1-\gamma)\sum_{t=1}^{+\infty}\gamma^{t-1}\mathbb{P}_{\pi^e}(s^{(t)} = s, a^{(t)} = a | s^{(1)} = \underline{s}), \quad \forall (s,a)\in\mathcal{S}\times\mathcal{A}.
\end{align*}
Here $\mathbb{P}_{\pi^e}$ is defined as similarly as $\mathbb{E}_{\pi^e}$ in Section \ref{subsec: data and eval}.
Our main result is the following theorem.

\begin{thm}[Non-asymptotic Analysis: Confounded MDP]\label{thm: non asymptotic mdp}
    Under Assumptions \ref{assump: linear}, \ref{assump: instrumental variable}, and \ref{assump: coverage},  the value function estimator \eqref{eq: hat V mdp} of $\pi^e$ satisfies that, with probability at least $1-\delta$,
\begin{align}\label{eq: non asymptotic mdp}
    \big|\widehat{V}_{\pi^e}(s_0) - V_{\pi^e}(s_0)\big|\leq \widetilde{\mathcal{O}}\Bigg(\sqrt{\frac{\mathbb{E}_{(s,a)\sim d_{\pi^e,s_0}}\left[\|\boldsymbol{\phi}(s,a)\|^2_{\boldsymbol{\Sigma}^{-1}}\right]d}{\underline{\sigma}(1-\gamma)^4n}}\,\Bigg).
\end{align}
for any initial state $s_0\in\mathcal{S}$, where $\widetilde{\mathcal{O}}(\cdot)$ hides universal constants, logarithm factors, and higher-order terms.
\end{thm}
\begin{proof}[Proof of Theorem \ref{thm: non asymptotic mdp}]
    See Appendix \ref{subsec: proof non asymptotic mdp} for a detailed proof.
\end{proof}
By Theorem \ref{thm: non asymptotic mdp}, the non-asymptotic upper bound \eqref{eq: non asymptotic mdp} depends on the term $\mathbb{E}_{(s,a)\sim d_{\pi^e,s_0}}[\|\boldsymbol{\phi}(s,a)\|^2_{\boldsymbol{\Sigma}^{-1}}]$.
This term characterizes how the offline dataset $\mathbb{D}$ covers the feature $\boldsymbol{\phi}(s,a)$ when averaged on the trajectories induced by $\pi^e$ starting from $s_0$, i.e., $(s,a)\sim d_{\pi^e,s_0}$.
With the same arguments as in \eqref{eq: cover bandit bound}, we can show that this factor is upper bouneded by a factor of $\underline{\sigma}^{-1}$.

Therefore by Theorem \ref{thm: non asymptotic mdp}, the two-stage estimator \eqref{eq: hat V mdp} enjoys a $\widetilde{\mathcal{O}}(\sqrt{d/\underline{\sigma}^2(1-\gamma)^4n})$ statistical rate.
To ensure an $\epsilon$-accuracy, the number of samples needed is approximately 
\begin{align}\label{eq: sample complexity}
    n=\widetilde{\mathcal{O}}\left(\frac{d}{\underline{\sigma}^2(1-\gamma)^4\epsilon^2}\right).
\end{align}
According to \eqref{eq: non asymptotic mdp}, up to logarithm factors and higher order terms on $n$, the statistical rate does not suffer from the issue of confoundedness. 
As is shown in the proof of Theorem \ref{thm: non asymptotic mdp}, the dependence on the size of instrumental variable space $\mathcal{Z}$ only appears in the logarithm factors and higher order terms.

\subsection{Asymptotic Analysis}
In this section, we establish the asymptotic distribution for the two-stage estimator \eqref{eq: hat V mdp}.
In order to present the result, in the sequel, we denote the vector of value function as $\boldsymbol{V}_{\pi^e}=(V_{\pi^e}(s))_{s\in\mathcal{S}}\in\mathbb{R}^S$.
Similarly, we define the vector $\widehat{\boldsymbol{V}}_{\pi^e} = (\widehat{V}_{\pi^e}(s))_{s\in\mathcal{S}}\in\mathbb{R}^S$ where $\widehat{V}_{\pi^e}(s)$ is given in \eqref{eq: hat V mdp}.
Besides, we consider the conditional expectation $\rho(s,a,|z)$ as an element in the space $\mathbb{R}^{S\times A\times Z}$.
Our main result is the following theorem.

\begin{thm}[Asymptotic Analysis: Confounded MDP]\label{thm: asymptotic mdp}
    Under Assumptions \ref{assump: linear}, \ref{assump: instrumental variable}, and \ref{assump: coverage}, by denoting $X_i:=(s_i, a_i, z_i, r_i, s_i^{\prime})\in\mathcal{X}$, it holds that the value function estimator \eqref{eq: hat V mdp} of $\pi^e$ is asymptotically linear, i.e.,
    \begin{align}\label{eq: asymptotically linear}
        \sqrt{n}\cdot\big(\widehat{\boldsymbol{V}}_{\pi^e}-\boldsymbol{V}_{\pi^e}\big)= \frac{1}{\sqrt{n}}\sum_{i=1}^n \boldsymbol{\Phi}_{\pi^e}^\top\big(\boldsymbol{h}(X_i;\rho)-\EB[ \boldsymbol{h}(X_i;\rho)]\big)+o_P(1),
    \end{align}
    where the matrix $\boldsymbol{\Phi}_{\pi^e}\in\RB^{d\times S}$ is defined as $\boldsymbol{\Phi}_{\pi^e} = (\boldsymbol{\phi}_{\pi^e}(s))_{s\in\mathcal{S}}$, and the asymptotic linear estimator $\boldsymbol{h}(X_i;\rho):\mathcal{X}\times\mathbb{R}^{S\times A\times Z}\mapsto\mathbb{R}^d$ is defined as 
    \begin{align}
        \boldsymbol{h}(X_i;\rho):=\boldsymbol{f}(X_i;\rho)+\nabla_{\rho} \EB [\boldsymbol{f}(X_i;\rho)] \,\nabla_{\EB [\boldsymbol{g}(X_i)]}\boldsymbol{d}(\EB [\boldsymbol{g}(X_i)]) \, \boldsymbol{g}(X_i).
    \end{align}
    Here $\boldsymbol{f}(X_i;\rho):\mathcal{X}\times\mathbb{R}^{S\times A\times Z}\mapsto\mathbb{R}^d$, $\boldsymbol{g}:\mathcal{X}\mapsto\mathbb{R}^{S\times A\times Z}$, and $\boldsymbol{d}:\mathbb{R}^{S\times A\times Z}\mapsto\mathbb{R}^{S\times A\times Z}$ are defined as
    \begin{align}
        \boldsymbol{f}(X_i;\rho):=\left(\boldsymbol{I}-\gamma \boldsymbol{A}\right)^{-1}\boldsymbol{\Sigma}^{-1}\boldsymbol{\phi}_{\rho}(z_i)\left(r_i+\gamma\boldsymbol{\phi}_{\pi^e}(s_i')^\top\boldsymbol{\theta}-\boldsymbol{\phi}_{\rho}(z_i)^\top\boldsymbol{\theta}\right),
    \end{align}
    \begin{align}
        \boldsymbol{g}^{(s,a,z)}(X_i):=\mathbbm{1}\{s_i=s, a_i=a, z_i=z\},\quad
        \boldsymbol{d}^{(s,a,z)}(\boldsymbol{x}):=\frac{\boldsymbol{x}^{(s,a,z)}}{\sum_{(s',a')\in\mathcal{S}\times\mathcal{A}}\boldsymbol{x}^{(s',a',z)}}, 
    \end{align}
    for any $X_i=(s_i,a_i,z_i,r_i,s_i^{\prime})\in\mathcal{X}$ and $\boldsymbol{x}\in\mathbb{R}^{S\times A\times Z}$ respectively.
\end{thm}
\begin{proof}[Proof of Theorem \ref{thm: asymptotic mdp}]
    See Appendix \ref{subsec: proof asymptotic mdp} for a detailed proof.
\end{proof}

In Theorem \ref{thm: asymptotic mdp}, there is a partial derivative $\nabla_{\rho}\EB [\boldsymbol{f}(X_i;\rho)]$ in the asymptotic linear estimator $\boldsymbol{h}(X_i;\rho)$. 
In this derivative, we let the factor $(\boldsymbol{I}-\gamma \boldsymbol{A})^{-1}\boldsymbol{\Sigma}^{-1}$ stay fixed, even this factor implicitly depends on $\rho$.

By Theorem \ref{thm: asymptotic mdp}, we can deduce that the asymptotic variance of the two-stage estimator \eqref{eq: hat V mdp} is given by 
\begin{align*}
    \boldsymbol{\Phi}_{\pi^e}^\top\mathbf{Var}[\boldsymbol{h}(X_1;\rho)]\boldsymbol{\Phi}_{\pi^e}.
\end{align*} 
We note that the asymptotic linear estimator $\boldsymbol{h}(X_i;\rho)$ has two components, i.e., 
\begin{align*}
    \boldsymbol{h}(X_i;\rho) = \boldsymbol{f}(X_i;\rho) + \big(\boldsymbol{h}(X_i;\rho)-\boldsymbol{f}(X_i;\rho)\big)
\end{align*}
The second part $\boldsymbol{h}(X_i;\rho)-\boldsymbol{f}(X_i;\rho)$ is caused by the first stage for estimating the conditional feature $\boldsymbol{\phi}_{\rho}(z)$, whose randomness may enlarge the asymptotic variance of the two-stage estimator.
Suppose we were given the exact value of $\rho$, we could obtain that the asymptotic variance of the two-stage estimator \eqref{eq: hat V mdp} is
\begin{align*}
    \boldsymbol{\Phi}_{\pi^e}^\top\mathbf{Var}[\boldsymbol{f}(X_1;\rho)]\boldsymbol{\Phi}_{\pi^e}.
\end{align*}
It is still unclear whether the two-stage estimator \eqref{eq: hat V mdp} can achieve the semiparametric efficiency lower bound \citep{yao2010efficient}. 
We leave the discussion of semiparametric efficiency to our future work. 

Finally, by the fact that $\sqrt{n}\cdot(\widehat{\boldsymbol{V}}_{\pi^e}-\boldsymbol{V}_{\pi^e})$ is asymptotically linear \eqref{eq: asymptotically linear}, we can apply the multidimensional Central Limit Theorem (CLT) to obtain asymptotic normality in the following corollary.

\begin{cor}\label{cor: normality}
Denote $X_i:=(s_i, a_i, r_i, s_i', z_i)$, in the same setting as Theorem~\ref{thm: asymptotic mdp}. Then the value estimator \eqref{eq: hat V mdp} is asymptotic normal:
\begin{align*}
    \sqrt{n}\cdot\big(\widehat{\boldsymbol{V}}_{\pi^e}-\boldsymbol{V}_{\pi^e}\big)\tod \NM\big(\boldsymbol{0},\boldsymbol{\Phi}_{\pi^e}^\top\mathbf{Var}[\boldsymbol{h}(X_1;\rho)]\boldsymbol{\Phi}_{\pi^e}\big).
\end{align*}
\end{cor}
\begin{proof}[Proof of Corollary \ref{cor: normality}]
    This directly follows from the asymptotic linearity \eqref{eq: asymptotically linear} and the above discussion.
\end{proof}

\section{Conclusions}
\label{sec: conclusion}
In this work, we present the first statistical result of OPE in confounded MDPs based on the tool of instrumental variables. We propose a two-stage estimator of the value function from the offline dataset, which is corrupted with observable confounders. In non-asymptotic viewpoint, we provide the two-stage estimator is close to the true value function with statistical rate $\widetilde{\OM}(n^{-1/2})$. In asymptotic viewpoint, we prove that the two-stage estimator is asymptotic normal with typical rate $n^{1/2}$, from which we open an approach  to statistical inference for confounded MDPs. However, there are some directions to extend our work. Firstly, how does one design a two-stage estimator when the instrumental variable is continuous. Secondly, it is still unknown whether our two-stage estimator achieves semiparametric efficiency bound. We leave these issues to  future work.

\bibliographystyle{plainnat}
\bibliography{refer.bib}

\newpage
\appendix
\section{Proof in Section \ref{sec: bandit}: Confouneded Multi-armed Bandit}\label{sec: proof bandit}
In this section, we prove the theoretical results of the two-stage estimator for confounded bandits (Section~\ref{sec: bandit}).

\subsection{Proof of Corollary \ref{cor: non asymptotic bandit}}\label{subsec: proof non asymptotic bandit}
Despite the fact that this result for confounded bandits is a corollary of Theorem \ref{thm: non asymptotic mdp} for confounded MDPs, we still prove this result from scratch. 
As we can see, the proof for the confounded bandit case serves as an important building block in the proof of the main result for confounded MDPs.

\begin{proof}[Proof of Corollary \ref{cor: non asymptotic bandit}]
    For notational simplicity, we define that
    \begin{align*}
        \widehat{\boldsymbol{\tau}} = \frac{1}{n}\sum_{i=1}^n\boldsymbol{\phi}_{\widehat{\rho}}(z_i)r_i,\quad \widehat{\boldsymbol{\Sigma}} = \frac{1}{n}\sum_{i=1}^n\boldsymbol{\phi}_{\widehat{\rho}}(z_i)\boldsymbol{\phi}_{\widehat{\rho}}(z_i)^\top, \quad \boldsymbol{\phi}_{\pi^e} = \sum_{a\in\mathcal{A}}\pi^e(a)\boldsymbol{\phi}(a).
    \end{align*}    
    With the above notations, we now upper bound the estimation error $|V_{\pi^e}-\widehat{V}_{\pi^e}|$ as 
    \begin{align}
        \big|V_{\pi^e}-\widehat{V}_{\pi^e}\big|&=\boldsymbol{\phi}_{\pi^e}(\widehat{\boldsymbol{w}}-\boldsymbol{w})\notag\\
        & = \boldsymbol{\phi}_{\pi^e}\widehat{\boldsymbol{\Sigma}}^{-1}(\widehat{\boldsymbol{\tau}}-\widehat{\boldsymbol{\Sigma}}\boldsymbol{w})\notag\\
        &\leq \|\boldsymbol{\phi}_{\pi^e}\|_{\widehat{\boldsymbol{\Sigma}}^{-1}}\cdot\|\widehat{\boldsymbol{\tau}}-\widehat{\boldsymbol{\Sigma}}\boldsymbol{w}\|_{\widehat{\boldsymbol{\Sigma}}^{-1}}\notag\\
        &\leq \underbrace{\|\boldsymbol{\phi}_{\pi^e}\|_{\boldsymbol{\Sigma}^{-1}}}_{\text{(i)}}
        \cdot\underbrace{\|\boldsymbol{\Sigma}^{\frac{1}{2}}\widehat{\boldsymbol{\Sigma}}^{-1}\boldsymbol{\Sigma}^{\frac{1}{2}}\|}_{\text{(ii)}}
        \cdot\underbrace{\|\widehat{\boldsymbol{\tau}}-\widehat{\boldsymbol{\Sigma}}\boldsymbol{w}\|_{\boldsymbol{\Sigma}^{-1}}}_{\text{(iii)}}\label{eq: error decomposition bandit}.
    \end{align}
    Here the term (i) characterizes how the covariance matrix of $\boldsymbol{\phi}_{\rho}(z_i)$, which is $\boldsymbol{\Sigma}$, covers the feature $\boldsymbol{\phi}_{\pi^e}$. This term is bounded by $\mathcal{O}(\underline{\sigma}^{-\frac{1}{2}})$ where $\underline{\sigma}$ is the minimal eigenvalue of $\boldsymbol{\Sigma}$.
    In the sequel, we bound the term (ii) and term (iii) in \eqref{eq: error decomposition bandit} respectively.

    \vspace{3mm}
    \noindent
    \textbf{Bound of term (ii).} 
    This term characterizes how well the covariance matrix $\boldsymbol{\Sigma}$ is estimated.
    In order to bound this term, we invoke Lemma \ref{lem:cov_estimation} and obtain that, with probability at least $1-\delta$, 
    \begin{align}
        \|\boldsymbol{\Sigma}^{1/2}\widehat{\boldsymbol{\Sigma}}^{-1}\boldsymbol{\Sigma}^{1/2}\|\leq \left(1-\sqrt{\frac{8\underline{\sigma}^{-2}d\log\left(d|\mathcal{S}||\mathcal{A}||\mathcal{Z}|/\delta\right)}{n}}+\frac{6\underline{\sigma}^{-1}d|\mathcal{Z}|\log\left(d|\mathcal{S}||\mathcal{A}||\mathcal{Z}|/\delta\right)}{n}\right)^{-1}.
    \end{align}
    That is, for $n$ large enough, term (ii) is of order $\mathcal{O}(1)$.

    \vspace{3mm}
    \noindent
    \textbf{Bound of term (iii).} 
    This term characterizes how well the parameter $\boldsymbol{w}$ is estimated.
    By the definition of $\boldsymbol{\tau}$ and $\boldsymbol{\Sigma}$, we have that 
    \begin{align*}
        \widehat{\boldsymbol{\tau}}-\widehat{\boldsymbol{\Sigma}}\boldsymbol{w} &= \frac{1}{n}\sum_{i=1}^n\boldsymbol{\phi}_{\widehat{\rho}}(z_i)\big(r_i - \boldsymbol{\phi}_{\widehat{\rho}}(z_i)^\top\boldsymbol{w}\big)\\
            &= \underbrace{\frac{1}{n}\sum_{i=1}^n\boldsymbol{\phi}_{\rho}(z_i)\big(r_i - \boldsymbol{\phi}_{\rho}(z_i)^\top\boldsymbol{w}\big)}_{\sharp} + \underbrace{\frac{1}{n}\sum_{i=1}^n\big(\boldsymbol{\phi}_{\widehat{\rho}}(z_i)-\boldsymbol{\phi}_{\rho}(z_i)\big)r_i - \big(\boldsymbol{\phi}_{\widehat{\rho}}(z_i)\boldsymbol{\phi}_{\widehat{\rho}}(z_i)^\top-\boldsymbol{\phi}_{\rho}(z_i)\boldsymbol{\phi}_{\rho}(z_i)^\top\big)\boldsymbol{w}}_{\dagger}
    \end{align*}
    In the sequel, we first bound the term $\sharp$. 
    To this end, we condition on the instrumental variables $z_1,\cdots,z_n$.
    Notice that for each $i$, it holds that 
    \begin{align*}
        \mathbb{E}\big[\boldsymbol{\phi}_{\rho}(z_i)\big(r_i - \boldsymbol{\phi}_{\rho}(z_i)^\top\boldsymbol{w}\big)\big| z_1,\cdots,z_n\big] = \mathbb{E}\big[\boldsymbol{\phi}_{\rho}(z_i)\big(r_i - \boldsymbol{\phi}_{\rho}(z_i)^\top\boldsymbol{w}\big)\big| z_i\big] = \boldsymbol{\phi}_{\rho}(z_i)\big(\mathbb{E}[r_i|z_i] - \boldsymbol{\phi}_{\rho}(z_i)^\top\boldsymbol{w}\big) = \boldsymbol{0}.
    \end{align*}
    Also, due to Assumption \ref{assump: linear} it holds that $|r_i|\leq 2$, $\|\boldsymbol{w}\|_2\leq 1$, and $\|\boldsymbol{\phi}_{\rho}(z_i)\|_2\leq 1$.
    As a result, by Hoeffding's inequality, we can obtain that with probability $\mathbb{P}(\cdot|z_1,\cdots,z_n)$ at least $1-\delta$, 
    \begin{align}\label{eq: bandit term iii sharp}
        \|\sharp\|_2= \left\|\frac{1}{n}\sum_{i=1}^n\boldsymbol{\phi}_{\rho}(z_i)\big(r_i - \boldsymbol{\phi}_{\rho}(z_i)^\top\boldsymbol{w}\big)\right\|_2 \leq \sqrt{\frac{4d\log\left(1/\delta\right)}{n}}.
    \end{align}
    Taking expectation w.r.t. the $z_1,\cdots,z_n$, we conclude that \eqref{eq: bandit term iii sharp} holds with probability $\mathbb{P}(\cdot)$ at least $1-\delta$.
    
    Besides, for the term $\dagger$, we invoke Lemma \ref{lem:concentration_empirical_feature} and obtain that 
    \begin{align}
        \|\dagger\|_2 &\leq \frac{1}{n}\sum_{i=1}^n\|\big(\boldsymbol{\phi}_{\widehat{\rho}}(z_i)-\boldsymbol{\phi}_{\rho}(z_i)\big)r_i\|_2 + \frac{1}{n}\sum_{i=1}^n\|\big(\boldsymbol{\phi}_{\widehat{\rho}}(z_i)\boldsymbol{\phi}_{\widehat{\rho}}(z_i)^\top-\boldsymbol{\phi}_{\rho}(z_i)\boldsymbol{\phi}_{\rho}(z_i)^\top\big)\boldsymbol{w}\|_2\notag\\
        &\leq \frac{1}{n}\sum_{i=1}^n\|\big(\boldsymbol{\phi}_{\widehat{\rho}}(z_i)-\boldsymbol{\phi}_{\rho}(z_i)\big)\|_2\cdot|r_i| + \frac{1}{n}\sum_{i=1}^n\|\big(\boldsymbol{\phi}_{\widehat{\rho}}(z_i)-\boldsymbol{\phi}_{\rho}(z_i)\big)\|_2 \cdot \|\big(\boldsymbol{\phi}_{\widehat{\rho}}(z_i) + \boldsymbol{\phi}_{\rho}(z_i)\big)^\top\boldsymbol{w}\|_2\notag\\
        &\leq \frac{4}{n}\sum_{i=1}^n\|\big(\boldsymbol{\phi}_{\widehat{\rho}}(z_i)-\boldsymbol{\phi}_{\rho}(z_i)\big)\|_2\leq \sqrt{\frac{48d\log\left(d|\mathcal{S}||\mathcal{A}||\mathcal{Z}|/\delta\right)}{n}}+\frac{16\sqrt{d}|\mathcal{Z}|\log\left(d|\mathcal{S}||\mathcal{A}||\mathcal{Z}|/\delta\right)}{n}\label{eq: bandit term iii dagger},
    \end{align}
    Now by combining the upper bounds \eqref{eq: bandit term iii sharp} and \eqref{eq: bandit term iii dagger} on the terms $\sharp$ and $\dagger$, we obtain that with probability at least $1-2\delta$, it holds that 
    \begin{align}\label{eq: bandit term iii},
        \text{(iii)} &= \|\widehat{\boldsymbol{\tau}}-\widehat{\boldsymbol{\Sigma}}\boldsymbol{w}\|_{\boldsymbol{\Sigma}^{-1}} \leq \underline{\sigma}^{-\frac{1}{2}}\cdot\|\widehat{\boldsymbol{\tau}}-\widehat{\boldsymbol{\Sigma}}\boldsymbol{w}\|_2 \leq \underline{\sigma}^{-\frac{1}{2}}\cdot\big( \|\sharp\|_2+\|\dagger\|_2\big)\notag\\
        &\leq 4\underline{\sigma}^{-\frac{1}{2}}\cdot \sqrt{\frac{7d\log\left(d|\mathcal{S}||\mathcal{A}||\mathcal{Z}|/\delta\right)}{n}}+\frac{16\underline{\sigma}^{-\frac{1}{2}}\sqrt{d}|\mathcal{Z}|\log\left(d|\mathcal{S}||\mathcal{A}||\mathcal{Z}|/\delta\right)}{n},
    \end{align}

    \noindent
    \textbf{Combining bounds on term (i), (ii), (iii).} 
    Finally, by combining the three upper bounds and omitting the terms of lower order, we can obtain that with probabiltiy at least $1-2\delta$,
    \begin{align*}
        \big|V_{\pi^e}-\widehat{V}_{\pi^e}\big| \lesssim \|\boldsymbol{\phi}_{\pi^e}\|_{\boldsymbol{\Sigma}^{-1}}\cdot \left(\sqrt{\frac{d\log\left(d|\mathcal{S}||\mathcal{A}||\mathcal{Z}|/\delta\right)}{\underline{\sigma} n}}+\frac{\sqrt{d}|\mathcal{Z}|\log\left(d|\mathcal{S}||\mathcal{A}||\mathcal{Z}|/\delta\right)}{\underline{\sigma}^{\frac{1}{2}}n}\right)\lesssim \widehat{\mathcal{O}}\left(  \|\boldsymbol{\phi}_{\pi^e}\|_{\boldsymbol{\Sigma}^{-1}}\cdot \sqrt{\frac{d}{\underline{\sigma} n}} \right),
    \end{align*}
    where $\widetilde{\mathcal{O}}$ hides universal constants and logarithm factors. This finishes the proof of Corollary \ref{cor: non asymptotic bandit}.
\end{proof}

\section{Proof in Section \ref{sec: mdp}: Confounded Markov Decision Process}\label{sec: proof mdp}
In this section, we prove the theoretical results of the two-stage estimator for confounded MDPs (Section~\ref{sec: mdp}).

\subsection{Proof of Proposition \ref{prop: identification policy value}}\label{subsec: proof identification policy value}
\begin{proof}[Proof of Proposition \ref{prop: identification policy value}]
    We first rewrite $V_{\pi^e}$ using the Bellman equation \citep{sutton2018reinforcement} as
    \begin{align}
        V_{\pi^e}(s)&=\mathbb{E}_{\pi^e}[r^{(1)}| s^{(1)}=s]+\gamma \mathbb{E}_{\pi^e}[V_{\pi^e}(s^{(2)})|s^{(1)}=s]\notag\\
        &=\sum_{a\in\AM}\pi^e(a|s)\boldsymbol{\phi}(s,a)^\top \boldsymbol{w}+\gamma\sum_{a\in\AM}\pi^e(a|s) \boldsymbol{\phi}(s,a)^\top\sum_{s^{\prime}\in\SM}\boldsymbol{\nu}(s^{\prime})V_{\pi^e}(s^{\prime}),\label{eq: bellman}
    \end{align}
    where in the first equality, we use the Bellman equation, and in the last equality, we use the linear structure in Assumption \ref{assump: linear} and the fact that $\mathbb{E}_{\pi^e}[\epsilon^{(1)}]=0$.
    For simplicity, we define a vector-valued function on $\mathcal{S}$ as 
    \begin{align*}
        \boldsymbol{\phi}_{\pi^e}(s)\coloneqq\sum_{a\in\AM}\boldsymbol{\phi}(s,a)\pi^e(a|s)\in\mathbb{R}^d.
    \end{align*}
    In order to solve the value function $V_{\pi^e}$, we further define a vector as
    \begin{align*}
        \boldsymbol{\beta}\coloneqq \sum_{s^{\prime}\in\SM}\boldsymbol{\nu}(s^{\prime})V_{\pi^e}(s^{\prime})\in\mathbb{R}^d.
    \end{align*}
    Using Equation \eqref{eq: bellman}, we can derive the following equation for the vector $\boldsymbol{\beta}$,
    \begin{align*}
        &\boldsymbol{\beta}=\sum_{s^{\prime}\in\SM}\boldsymbol{\nu}(s^{\prime})\boldsymbol{\phi}_{\pi^e}(s^{\prime})^\top\left(\boldsymbol{w}+\gamma \boldsymbol{\beta}\right)\quad \Rightarrow\quad  \boldsymbol{\beta}=(\boldsymbol{I}-\gamma \boldsymbol{A})^{-1}\boldsymbol{A}\boldsymbol{w},\quad \textnormal{with }\boldsymbol{A}\coloneqq\sum_{s\in\SM}\boldsymbol{\nu}(s)\boldsymbol{\phi}_{\pi^e}(s)^\top.
    \end{align*}
    We note that here $\boldsymbol{I}-\gamma \boldsymbol{A}$ is invertible due to the boundedness assumption in Assumption \ref{assump: linear}.
    Consequently, we can rewrite $V_{\pi^e}(s)$ as a linear function with respect to the feature mapping $\boldsymbol{\phi}_{\pi^e}(s)$,
    \begin{align}
        \label{eq: V mdp}
        V_{\pi^e}(s)=\boldsymbol{\phi}_{\pi^e}(s)^\top\boldsymbol{\theta},\quad \textnormal{with } \boldsymbol{\theta}\coloneqq\boldsymbol{w}+\gamma\boldsymbol{\beta}=(\boldsymbol{I}-\gamma \boldsymbol{A})^{-1}\boldsymbol{w}.
    \end{align}
    As a result, in order to identify the policy value $V_{\pi^e}$, it suffices to identify the unknown vector $\boldsymbol{w}$ and matrix $\boldsymbol{A}$ using observational data.
    Using the same technique as in \ref{sec: bandit}, we can show that 
    \begin{align}
        \boldsymbol{w} = \mathbb{E}[\boldsymbol{\phi}_{\rho}(z_i)\boldsymbol{\phi}_{\rho}(z_i)^\top]^{-1}\mathbb{E}[\boldsymbol{\phi}_{\rho}(z_i)r_i],\quad  \boldsymbol{A} =\mathbb{E}[\boldsymbol{\phi}_\rho(z_i)\boldsymbol{\phi}_\rho(z_i)^\top]^{-1}\mathbb{E}[\boldsymbol{\phi}_\rho(z_i)\boldsymbol{\phi}_{\pi^e}(s_i^{\prime})^\top]. 
    \end{align}
    This concludes the proof of Proposition \ref{prop: identification policy value}.
\end{proof}

\subsection{Proof of Theorem \ref{thm: non asymptotic mdp}}\label{subsec: proof non asymptotic mdp}
We first define the action-value function, i.e., the $Q$-function, and the visitation measure introduced by the target policy $\pi^e$.
Specifically, we define the $Q$-function of $\pi^e$ as 
\begin{align}
    Q_{\pi^e}(s,a) = \mathbb{E}_{\pi^{e}}\left[\sum_{t=1}^{+\infty}\gamma^{t-1}r^{(t)}\middle| s^{(1)}=s, a^{(1)}=a\right],\quad \forall (s,a)\in\mathcal{S}\times\mathcal{A}.
\end{align}
Following the same proof for Proposition \ref{prop: identification policy value}, we can prove that $Q_{\pi^e}$ is linear in $\boldsymbol{\phi}(s,a)$. 
That is, $Q_{\pi^e}(s,a) = \boldsymbol{\phi}(s,a)^\top \boldsymbol{\theta}$ for the same parameter $\boldsymbol{\theta}$ as in Proposition \ref{prop: identification policy value}.
Then we define the visitation measure $d_{\pi^e,\underline{s}}(s,a)$ of $\pi^e$ as the cumulative probability that the agent visits $(s, a)$, i.e., 
\begin{align*}
    d_{\pi^e,\underline{s}}(s,a) = (1-\gamma)\sum_{t=1}^{+\infty}\gamma^{t-1}\mathbb{P}_{\pi^e}(s^{(t)} = s, a^{(t)} = a | s^{(1)} = \underline{s}), \quad \forall (s,a)\in\mathcal{S}\times\mathcal{A}.
\end{align*}
Using the notion of visitation measure, one can rewrite the value function $V_{\pi^e}(s)$ as 
\begin{align*}
    V_{\pi^e}(s_0) = \frac{1}{1-\gamma}\mathbb{E}_{(s,a)\sim d_{\pi^e,s_0}}[R(s,a)],\quad \forall s_0\in\mathcal{S}.
\end{align*}
Our analysis depends on an important evaluation error lemma \citep{xie2020q} (Lemma \ref{lem: evaluation error}), which we state in the following.
For any function $Q:\mathcal{S}\times\mathcal{A}\mapsto\mathbb{R}$ and $V=\mathbb{E}_{a\sim \pi^e(\cdot|s)}[Q(s,a)]:\mathcal{S}\mapsto\mathbb{R}$, 
\begin{align}\label{eq: evaluation error}
    V(s_0) - V_{\pi^e}(s_0) = \frac{1}{1-\gamma}\cdot\mathbb{E}_{(s,a)\sim d_{\pi^e,s_0}}\left[Q(s,a) - R(s,a) - \gamma \mathbb{E}_{s^{\prime}\sim P(\cdot|s,a),a^{\prime}\sim \pi^e(\cdot|s^{\prime})}[V(s^{\prime})]\right].
\end{align}

\begin{proof}[Proof of Theorem \ref{thm: non asymptotic mdp}]
    Recall that the estimator is given by $\widehat{V}_{\pi^e}(s) = \boldsymbol{\phi}_{\pi^e}(s)^\top\widehat{\boldsymbol{\theta}}$.
    We also define $\widehat{Q}_{\pi^e}(s,a) = \boldsymbol{\phi}(s,a)^\top\widehat{\boldsymbol{\theta}}$.
    Now by taking $V = \widehat{V}_{\pi^e}$ and $Q = \widehat{Q}_{\pi^e}$ in \eqref{eq: evaluation error}, we can obtain that 
    \begin{align}
        \widehat{V}_{\pi^e}(s_0) - V_{\pi^e}(s_0) &= \frac{1}{1-\gamma}\cdot\mathbb{E}_{(s,a)\sim d_{\pi^e,s_0}}\left[\widehat{Q}_{\pi^e}(s,a) - R(s,a) - \gamma \mathbb{E}_{s^{\prime}\sim P(\cdot|s,a),a^{\prime}\sim \pi^e(\cdot|s^{\prime})}[\widehat{V}_{\pi^e}(s^{\prime})]\right]\notag\\
        & = \frac{1}{1-\gamma}\cdot\mathbb{E}_{(s,a)\sim d_{\pi^e,s_0}}\left[\boldsymbol{\phi}(s,a)^\top\widehat{\boldsymbol{\theta}} - \boldsymbol{\phi}(s,a)^\top\boldsymbol{w} - \gamma \sum_{s^{\prime}\in\mathcal{S}}\boldsymbol{\phi}(s,a)^\top\boldsymbol{\nu}(s^{\prime})\boldsymbol{\phi}_{\pi^e}(s^{\prime})^\top\widehat{\boldsymbol{\theta}}\right]\notag\\
        & = \frac{1}{1-\gamma}\cdot\mathbb{E}_{(s,a)\sim d_{\pi^e,s_0}}\Big[\boldsymbol{\phi}(s,a)^\top\underbrace{\big(\widehat{\boldsymbol{\theta}} - \boldsymbol{w} - \gamma \boldsymbol{A}\widehat{\boldsymbol{\theta}}\big)}_{(\star)}\Big]\label{eq: proof non asymptotic mdp A}
    \end{align}
    where in \eqref{eq: proof non asymptotic mdp A} we use the definition that $\boldsymbol{A} = \sum_{s^{\prime}\in\mathcal{S}}\boldsymbol{\nu}(s^{\prime})\boldsymbol{\phi}_{\pi^e}(s^{\prime})^\top$.
    Following the same notations as in the proof of Proposition \ref{prop: identification policy value}, and using the definition of $\widehat{\boldsymbol{\theta}}$ in \eqref{eq: hat V mdp}, we can rewrite the term $(\star)$ as
    \begin{align*}
        (\star) & = (\boldsymbol{I} - \gamma \boldsymbol{A})\widehat{\boldsymbol{\theta}} - \boldsymbol{w}
         = (\boldsymbol{I} - \gamma \boldsymbol{A})(\boldsymbol{I} - \gamma \widehat{\boldsymbol{A}})^{-1}\widehat{\boldsymbol{w}} - \widehat{\boldsymbol{w}} + \widehat{\boldsymbol{w}} - \boldsymbol{w} = (\widehat{\boldsymbol{A}} -\boldsymbol{A})(\boldsymbol{I} - \gamma \widehat{\boldsymbol{A}})^{-1}\widehat{\boldsymbol{w}} + (\widehat{\boldsymbol{w}} - \boldsymbol{w}).
    \end{align*}
    Using the definition of $\widehat{\boldsymbol{A}}$ and $\widehat{\boldsymbol{w}}$ in \eqref{eq: hat alpha A}, we can then bound the right hand side of \eqref{eq: proof non asymptotic mdp A} as
    \begin{align*}
        \eqref{eq: proof non asymptotic mdp A} & = \frac{1}{1-\gamma}\cdot\mathbb{E}_{(s,a)\sim d_{\pi^e,s_0}}\Big[\boldsymbol{\phi}(s,a)^\top\widehat{\boldsymbol{\Sigma}}^{-1}\big((\widehat{\boldsymbol{B}} -\widehat{\boldsymbol{\Sigma}}\boldsymbol{A})(\boldsymbol{I} - \gamma \widehat{\boldsymbol{A}})^{-1}\widehat{\boldsymbol{w}} + (\widehat{\boldsymbol{\tau}} - \widehat{\boldsymbol{\Sigma}}\boldsymbol{w})\big)\Big]\\
        & \leq \frac{1}{1-\gamma}\cdot\mathbb{E}_{(s,a)\sim d_{\pi^e,s_0}}\left[\|\boldsymbol{\phi}(s,a)\|_{\widehat{\boldsymbol{\Sigma}}^{-1}}\cdot \left(\big\|(\widehat{\boldsymbol{B}} -\widehat{\boldsymbol{\Sigma}}\boldsymbol{A})(\boldsymbol{I} - \gamma \widehat{\boldsymbol{A}})^{-1}\widehat{\boldsymbol{w}}\big\|_{\widehat{\boldsymbol{\Sigma}}^{-1}} + \big\|\widehat{\boldsymbol{\tau}} - \widehat{\boldsymbol{\Sigma}}\boldsymbol{w}\big\|_{\widehat{\boldsymbol{\Sigma}}^{-1}}\right)\right]\\
        & \leq \frac{1}{1-\gamma}\cdot\underbrace{\mathbb{E}_{(s,a)\sim d_{\pi^e,s_0}}\left[\|\boldsymbol{\phi}(s,a)\|_{\boldsymbol{\Sigma}^{-1}}\right]}_{\text{(i)}} \cdot \underbrace{\|\boldsymbol{\Sigma}^{\frac{1}{2}}\widehat{\boldsymbol{\Sigma}}^{-1}\boldsymbol{\Sigma}^{\frac{1}{2}}\|}_{\text{(ii)}} \cdot \underbrace{\left(\big\|(\widehat{\boldsymbol{B}} -\widehat{\boldsymbol{\Sigma}}\boldsymbol{A})(\boldsymbol{I} - \gamma \widehat{\boldsymbol{A}})^{-1}\widehat{\boldsymbol{w}}\big\|_{\boldsymbol{\Sigma}^{-1}} + \big\|\widehat{\boldsymbol{\tau}} - \widehat{\boldsymbol{\Sigma}}\boldsymbol{w}\big\|_{\boldsymbol{\Sigma}^{-1}}\right)}_{\text{(iii)}}
    \end{align*} 
    Similar to the confounded bandit case, the term (i) characterizes how the covariance matrix of $\boldsymbol{\phi}_{\rho}(z_i)$, which is $\boldsymbol{\Sigma}$, covers the feature $\boldsymbol{\phi}(s,a)$, averaged by the visitation measure $d_{\pi^e,s_0}$. This term is bounded by $\mathcal{O}(\underline{\sigma}^{-\frac{1}{2}})$ where $\underline{\sigma}$ is the minimal eigenvalue of $\boldsymbol{\Sigma}$.
    Besides, the term (ii) is a bounded term of $\mathcal{O}(1)$ for $n$ large enough, which is guaranteed by Lemma \ref{lem:cov_estimation}.
    In the sequel, we upper bound the term (iii).

    Note that in the proof of Corollary \ref{cor: non asymptotic bandit}, we have derived an upper bound on the second term $\|\widehat{\boldsymbol{\tau}} - \widehat{\boldsymbol{\Sigma}}\boldsymbol{w}\|_{\boldsymbol{\Sigma}^{-1}}$ in term (iii) (See Appendix \ref{subsec: proof non asymptotic bandit}).
    Therefore, it suffices to derive the upper bound for the first term in term (iii), which we denote as the term (iii.a).
    Consider that 
    \begin{align}\label{eq: proof non asymptotic mdp term iii a}
        \text{(iii.a)} & \leq \underline{\sigma}^{-\frac{1}{2}}\cdot \|\widehat{\boldsymbol{B}} -\widehat{\boldsymbol{\Sigma}}\boldsymbol{A}\|\cdot \|(\boldsymbol{I} - \gamma \widehat{\boldsymbol{A}})^{-1}\|\cdot \|\widehat{\boldsymbol{w}}\|_2
    \end{align}
    Here the term $\|(\boldsymbol{I} - \gamma \widehat{\boldsymbol{A}})^{-1}\|$ is upper bounded by  
    \begin{align}\label{eq: proof non asymptotic mdp term iii a I - gamma A}
        \|(\boldsymbol{I} - \gamma \widehat{\boldsymbol{A}})^{-1}\| &= \big\|\boldsymbol{I} - \gamma \widehat{\boldsymbol{A}}\big\|^{-1} \leq  \big(\|\boldsymbol{I} - \gamma \boldsymbol{A}\| - \gamma\|\widehat{\boldsymbol{A}} - \boldsymbol{A}\|\big)^{-1}\notag\\
        &\leq \big(\|\boldsymbol{I} - \gamma \boldsymbol{A}\| - \gamma\|\widehat{\boldsymbol{B}} - \widehat{\boldsymbol{\Sigma}}\boldsymbol{A}\|\cdot \|\widehat{\boldsymbol{\Sigma}}^{-1}\|\big)^{-1}\notag\\
        &\leq \big(1-\gamma - \gamma\|\widehat{\boldsymbol{B}} - \widehat{\boldsymbol{\Sigma}}\boldsymbol{A}\|\cdot \|\boldsymbol{\Sigma}^{-1}\|\cdot \|\boldsymbol{\Sigma}^{\frac{1}{2}}\widehat{\boldsymbol{\Sigma}}^{-1}\boldsymbol{\Sigma}^{\frac{1}{2}}\|\big)^{-1}.
    \end{align}
    Similarly, the term $\|\widehat{\boldsymbol{w}}\|_2$ is upper bounded by 
    \begin{align}\label{eq: proof non asymptotic mdp term iii a w}
        \|\widehat{\boldsymbol{w}}\|_2 \leq \|\boldsymbol{w}\|_2 + \|\widehat{\boldsymbol{w}} - \boldsymbol{w}\|_2 \leq 1 + \|\widehat{\boldsymbol{\tau}} - \widehat{\boldsymbol{\Sigma}}\boldsymbol{w}\|\cdot \|\boldsymbol{\Sigma}^{-1}\|\cdot \|\boldsymbol{\Sigma}^{\frac{1}{2}}\widehat{\boldsymbol{\Sigma}}^{-1}\boldsymbol{\Sigma}^{\frac{1}{2}}\|
    \end{align}
    Now we are going to bound the term $\|\widehat{\boldsymbol{B}} - \widehat{\boldsymbol{\Sigma}}\boldsymbol{A}\|$.
    By the definition of $\widehat{\boldsymbol{B}}$ and $\widehat{\boldsymbol{\Sigma}}$, we have that 
    \begin{align}
        \widehat{\boldsymbol{B}} - \widehat{\boldsymbol{\Sigma}}\boldsymbol{A} &= \frac{1}{n}\sum_{i=1}^n\boldsymbol{\phi}_{\widehat{\rho}}(z_i)\big(\boldsymbol{\phi}_{\pi^e}(s_i^{\prime})^\top - \boldsymbol{\phi}_{\widehat{\rho}}(z_i)^\top\boldsymbol{A}\big)\notag\\
        &= \underbrace{\frac{1}{n}\sum_{i=1}^n\boldsymbol{\phi}_{\rho}(z_i)\big(\boldsymbol{\phi}_{\pi^e}(s_i^{\prime})^\top - \boldsymbol{\phi}_{\rho}(z_i)^\top\boldsymbol{A}\big)}_{\sharp} \notag\\
        &\quad\quad + \underbrace{\frac{1}{n}\sum_{i=1}^n\big(\boldsymbol{\phi}_{\widehat{\rho}}(z_i)-\boldsymbol{\phi}_{\rho}(z_i)\big)\boldsymbol{\phi}_{\pi^e}(s_i^{\prime})^\top - \big(\boldsymbol{\phi}_{\widehat{\rho}}(z_i)\boldsymbol{\phi}_{\widehat{\rho}}(z_i)^\top-\boldsymbol{\phi}_{\rho}(z_i)\boldsymbol{\phi}_{\rho}(z_i)^\top\big)\boldsymbol{A}}_{\dagger}.\label{eq: proof non asymptotic mdp decompose iii a}
    \end{align}
    To bound the term $\sharp$ in \eqref{eq: proof non asymptotic mdp decompose iii a}, we condition on the instrumental variables $z_1,\cdots,z_n$. 
    The conditional expectation of the summand given the instrumental variables are given by 
    \begin{align*}
        \mathbb{E}\big[\boldsymbol{\phi}_{\rho}(z_i)\big(\boldsymbol{\phi}_{\pi^e}(s_i^{\prime})^\top - \boldsymbol{\phi}_{\rho}(z_i)^\top\boldsymbol{A}\big)\big| z_1,\cdots,z_n\big] &= \mathbb{E}\big[\boldsymbol{\phi}_{\rho}(z_i)\big(\boldsymbol{\phi}_{\pi^e}(s_i^{\prime})^\top - \boldsymbol{\phi}_{\rho}(z_i)^\top\boldsymbol{A}\big)\big| z_i\big] \\
        &= \boldsymbol{\phi}_{\rho}(z_i)\big(\mathbb{E}[\boldsymbol{\phi}_{\pi^e}(s_i^{\prime})^\top|z_i] - \boldsymbol{\phi}_{\rho}(z_i)^\top\boldsymbol{A}\big) \\
        &= \boldsymbol{\phi}_{\rho}(z_i)\boldsymbol{\phi}_{\rho}(z_i)^\top\sum_{s^{\prime}\in\mathcal{S}}\boldsymbol{\nu}(s^{\prime})\boldsymbol{\phi}_{\pi^e}(s^{\prime}) - \boldsymbol{\phi}_{\rho}(z_i)\boldsymbol{\phi}_{\rho}(z_i)^\top\boldsymbol{A} \\
        &= \boldsymbol{0}.
    \end{align*}
    Also, according to Assumption \ref{assump: linear} it holds that $\|\boldsymbol{\phi}_{\pi^e}(s_i^{\prime})\|_2\leq 1$, $\|\boldsymbol{\phi}_{\rho}(z_i)\|_2\leq 1$, and $\|\boldsymbol{A}\|\leq 1$.
    As a result, by matrix Bernstein inquality (Lemma \ref{lem:matrix_bernstein}), we can obtain that with probability $\mathbb{P}(\cdot|z_1,\cdots,z_n)$ at least $1-\delta$, 
    \begin{align}\label{eq: proof non asymptotic mdp term iii sharp}
        \|\sharp\|= \left\|\frac{1}{n}\sum_{i=1}^n\boldsymbol{\phi}_{\rho}(z_i)\big(\boldsymbol{\phi}_{\pi^e}(s_i^{\prime})^\top - \boldsymbol{\phi}_{\rho}(z_i)^\top\boldsymbol{A}\big)\right\| \leq \sqrt{\frac{8\log\left(d/\delta\right)}{n}} + \frac{2\log\left(d/\delta\right)}{3n}.
    \end{align}
    Taking expectation with respect to the instrumental variables $z_1,\cdots,z_n$, we conclude that \eqref{eq: proof non asymptotic mdp term iii sharp} holds with probability $\mathbb{P}(\cdot)$ at least $1-\delta$.
    Then for the term $\dagger$ in \eqref{eq: proof non asymptotic mdp decompose iii a}, we use Lemma \ref{lem:concentration_empirical_feature} and follow the same argument as in the proof of Corollary \ref{cor: non asymptotic bandit} (see Equation \eqref{eq: bandit term iii dagger}), which gives that with probability at least $1-\delta$.
    \begin{align}\label{eq: proof non asymptotic mdp term iii dagger}
        \|\dagger\|\leq \frac{3}{n}\sum_{i=1}^n\|\big(\boldsymbol{\phi}_{\widehat{\rho}}(z_i)-\boldsymbol{\phi}_{\rho}(z_i)\big)\|_2\leq \sqrt{\frac{27d\log\left(d|\mathcal{S}||\mathcal{A}||\mathcal{Z}|/\delta\right)}{n}}+\frac{12\sqrt{d}|\mathcal{Z}|\log\left(d|\mathcal{S}||\mathcal{A}||\mathcal{Z}|/\delta\right)}{n}.
    \end{align}
    By combining the bounds \eqref{eq: proof non asymptotic mdp term iii sharp} and \eqref{eq: proof non asymptotic mdp term iii dagger} on the term $\sharp$ and $\dagger$, with probability at least $1-2\delta$, it holds that
    \begin{align}\label{eq: proof non asymptotic mdp term iii a B - Sigma A}
        \|\widehat{\boldsymbol{B}} - \widehat{\boldsymbol{\Sigma}}\boldsymbol{A}\|  \leq \|\sharp\| +  \|\dagger\| \leq \sqrt{\frac{81d\log\left(d|\mathcal{S}||\mathcal{A}||\mathcal{Z}|/\delta\right)}{n}}+\frac{13\sqrt{d}|\mathcal{Z}|\log\left(d|\mathcal{S}||\mathcal{A}||\mathcal{Z}|/\delta\right)}{n}.
    \end{align}
    
    Now we are ready to bound the term (iii.a). According to \eqref{eq: proof non asymptotic mdp term iii a I - gamma A} and \eqref{eq: proof non asymptotic mdp term iii a B - Sigma A}, we know that for $n$ large enough, the term $\|(\boldsymbol{I} - \gamma \widehat{\boldsymbol{A}})^{-1}\|\leq 2/(1-\gamma)$.
    Similarly, according to \eqref{eq: proof non asymptotic mdp term iii a w} and \eqref{eq: bandit term iii}, for $n$ large enough the term $\|\widehat{\boldsymbol{w}}\|_2$ is $\mathcal{O}(1)$.
    Therefore, by combining these two bounds with \eqref{eq: proof non asymptotic mdp term iii a} and \eqref{eq: proof non asymptotic mdp term iii a B - Sigma A}, we can conclude that
    \begin{align*}
        \text{(iii.a)}\leq \frac{1}{\underline{\sigma}^{\frac{1}{2}}(1-\gamma)}\cdot\bigg(\sqrt{\frac{81d\log\left(d|\mathcal{S}||\mathcal{A}||\mathcal{Z}|/\delta\right)}{n}}+\frac{13\sqrt{d}|\mathcal{Z}|\log\left(d|\mathcal{S}||\mathcal{A}||\mathcal{Z}|/\delta\right)}{n}\bigg)
    \end{align*}
    with probability at least $1-2\delta$. Furthermore, by invoking the upper bounds on the second term in the term (iii), which is derived in \eqref{eq: bandit term iii} in the proof of Corollary \ref{cor: non asymptotic bandit}, we have that with probability at least $1-4\delta$,
    \begin{align*}
        \text{(iii)} &= \big\|(\widehat{\boldsymbol{B}} -\widehat{\boldsymbol{\Sigma}}\boldsymbol{A})(\boldsymbol{I} - \gamma \widehat{\boldsymbol{A}})^{-1}\widehat{\boldsymbol{w}}\big\|_{\boldsymbol{\Sigma}^{-1}} + \big\|\widehat{\boldsymbol{\tau}} - \widehat{\boldsymbol{\Sigma}}\boldsymbol{w}\big\|_{\boldsymbol{\Sigma}^{-1}}\\
        & \leq 4\underline{\sigma}^{-\frac{1}{2}}\cdot \sqrt{\frac{7d\log\left(d|\mathcal{S}||\mathcal{A}||\mathcal{Z}|/\delta\right)}{n}}+\frac{16\underline{\sigma}^{-\frac{1}{2}}\sqrt{d}|\mathcal{Z}|\log\left(d|\mathcal{S}||\mathcal{A}||\mathcal{Z}|/\delta\right)}{n}\\
        &\quad\quad + \frac{1}{\underline{\sigma}^{\frac{1}{2}}(1-\gamma)}\cdot\sqrt{\frac{81d\log\left(d|\mathcal{S}||\mathcal{A}||\mathcal{Z}|/\delta\right)}{n}}+\frac{13\sqrt{d}|\mathcal{Z}|\log\left(d|\mathcal{S}||\mathcal{A}||\mathcal{Z}|/\delta\right)}{\underline{\sigma}^{\frac{1}{2}}(1-\gamma)n}.
    \end{align*}
    Finally, by combining the term (i), (ii), and (iii), we have that with probability at least $1-4\delta$, 
    \begin{align*}
        \big|\widehat{V}_{\pi^e}(s_0) - V_{\pi^e}(s_0)\big| &\leq \frac{\mathbb{E}_{(s,a)\sim d_{\pi^e,s_0}}\left[\|\boldsymbol{\phi}(s,a)\|_{\boldsymbol{\Sigma}^{-1}}\right]}{\underline{\sigma}^{\frac{1}{2}}(1-\gamma)}\cdot\bigg(4\sqrt{\frac{7d\log\left(d|\mathcal{S}||\mathcal{A}||\mathcal{Z}|/\delta\right)}{n}}+\frac{16\sqrt{d}|\mathcal{Z}|\log\left(d|\mathcal{S}||\mathcal{A}||\mathcal{Z}|/\delta\right)}{n}\\
        &\quad\quad + \frac{1}{(1-\gamma)}\cdot\sqrt{\frac{81d\log\left(d|\mathcal{S}||\mathcal{A}||\mathcal{Z}|/\delta\right)}{n}}+\frac{13\sqrt{d}|\mathcal{Z}|\log\left(d|\mathcal{S}||\mathcal{A}||\mathcal{Z}|/\delta\right)}{(1-\gamma)n}\bigg)\\
        &\lesssim\widetilde{\mathcal{O}}\Bigg(\frac{\mathbb{E}_{(s,a)\sim d_{\pi^e,s_0}}\left[\|\boldsymbol{\phi}(s,a)\|_{\boldsymbol{\Sigma}^{-1}}\right]}{\underline{\sigma}^{\frac{1}{2}}(1-\gamma)^2}\cdot\sqrt{\frac{d}{n}}\Bigg),
    \end{align*}
    where $\widetilde{\mathcal{O}}$ hides universal constants and logarithm factors. This finishes the proof of Theorem \ref{thm: non asymptotic mdp}.
\end{proof}

\subsection{Proof of Theorem \ref{thm: asymptotic mdp}}\label{subsec: proof asymptotic mdp}
\begin{proof}[Proof of Theorem \ref{thm: asymptotic mdp}]
    From the linear representation of $V_{\pi^e}(s)$ and $\widehat{V}_{\pi^e}(s)$ in \eqref{eq: identification policy value} and \eqref{eq: hat V mdp}, we have that 
    \begin{align}
        \sqrt{n}\cdot\big(\widehat{\boldsymbol{V}}_{\pi^e}-\boldsymbol{V}_{\pi^e}\big) = \sqrt{n}\cdot\boldsymbol{\Phi}_{\pi^e}^\top(\widehat{\boldsymbol{\theta}} - \boldsymbol{\theta})
    \end{align}
    Therefore, we focus on proving the asymptotic normality of $\sqrt{n}\cdot(\widehat{\boldsymbol{\theta}}-\boldsymbol{\theta})$.
    We start from a decomposition of this term. 
    Recall that the expression of $\boldsymbol{\theta}$ in \eqref{eq: identification policy value} is
    \begin{align*}
        \boldsymbol{\theta} &= (\boldsymbol{I}-\gamma \boldsymbol{A})^{-1}\boldsymbol{w} = (\boldsymbol{I} - \gamma\boldsymbol{\Sigma}^{-1}\boldsymbol{B})^{-1}\boldsymbol{\Sigma}^{-1}\boldsymbol{\tau}:=\boldsymbol{G}^{-1}\boldsymbol{\tau},
    \end{align*}
    where we denote $\boldsymbol{G} = \boldsymbol{\Sigma}(\boldsymbol{I} - \gamma\boldsymbol{\Sigma}^{-1}\boldsymbol{B})$.
    Similarly, by defining $\widehat{\boldsymbol{G}} = \widehat{\boldsymbol{\Sigma}}(\boldsymbol{I} - \gamma\widehat{\boldsymbol{\Sigma}}^{-1}\widehat{\boldsymbol{B}})$, we can represent $\widehat{\boldsymbol{\theta}}$ as 
    \begin{align*}
        \widehat{\boldsymbol{\theta}} = \widehat{\boldsymbol{G}}^{-1}\widehat{\boldsymbol{\tau}}.
    \end{align*}
    Thus, we can decompose the term $\sqrt{n}\cdot(\widehat{\boldsymbol{\theta}}-\boldsymbol{\theta})$ by the following,
    \begin{align}\label{eq: decomposition asymptotic mdp}
        \sqrt{n}\cdot(\widehat{\boldsymbol{\theta}}-\boldsymbol{\theta})&=\sqrt{n}\cdot(\widehat{\boldsymbol{G}}^{-1}\widehat{\boldsymbol{\tau}} - \boldsymbol{G}^{-1}\boldsymbol{\tau})\notag\\
        & = \underbrace{\sqrt{n}\cdot\boldsymbol{G}^{-1}(\widehat{\boldsymbol{\tau}} - \boldsymbol{\tau})}_{\boldsymbol{\Delta}_1} + \underbrace{\sqrt{n}\cdot(\widehat{\boldsymbol{G}}^{-1} - \boldsymbol{G}^{-1})\boldsymbol{\tau}}_{\boldsymbol{\Delta}_2} + \underbrace{\sqrt{n}\cdot(\widehat{\boldsymbol{G}}^{-1} - \boldsymbol{G}^{-1})(\widehat{\boldsymbol{\tau}} - \boldsymbol{\tau})}_{\boldsymbol{\Delta}_3}
    \end{align}
    In the following parts, we analyze the properties of the term $\boldsymbol{\Delta}_1$, $\boldsymbol{\Delta}_2$ and $\boldsymbol{\Delta}_3$ respectively.

    \vspace{3mm}
    \noindent
    \textbf{Analysis of the term $\boldsymbol{\Delta_1}$.}
    To prove the asymptotic normality of $\boldsymbol{\Delta}_1$, we need to invoke some basic results from empirical process (see \cite{van2000asymptotic}). 
    Prior to that, to simplify the notations, we denote 
    \begin{align*}
        \EB_n[f]:=\frac{1}{n}\sum_{i=1}^n f(X_i), \quad \EB [f]:=\EB [f(X_1)],
    \end{align*}
    for any function $f$ on $\mathcal{X}$, where $X_i = (s_i,a_i,z_i,r_i,s_i^{\prime})\in\mathcal{X}$ are defined in Theorem \ref{thm: asymptotic mdp}.
    Also, we define that
    \begin{align*}
        \sG_n [f]=\sqrt{n}\cdot\left(\EB_n [f]-\EB [f]\right), 
    \end{align*}
    where the expectation $\EB[f]$ only depends on the randomness of $X_i$ and ignores the randomness in $f$.
    
    \begin{lem}\label{lem: delta1}
    We denote $\boldsymbol{f}_1(X_i;\rho):\mathcal{X}\times\mathbb{R}^{S\times A\times Z}\mapsto\mathbb{R}^d$ as 
    \begin{align*}
        \boldsymbol{f}_1(X_i;\rho):=\boldsymbol{G}^{-1}\boldsymbol{\phi}_{\rho}(z_i)r_i.
    \end{align*}
    Then the term $\boldsymbol{\Delta}_1$ satisfies that 
    \begin{align*}
        \boldsymbol{\Delta}_1 = \sG_n [\boldsymbol{f}_1(\rho)] +\sqrt{n}\cdot\big(\EB [\boldsymbol{f}_1(\widehat{\rho}\,)]-\EB [\boldsymbol{f}_1(\rho)]\big)+o_P(1),
    \end{align*}
    where by $\mathbb{E}[\boldsymbol{f}_1(\rho)]$ we mean $\mathbb{E}_{X_1}[\boldsymbol{f}_1(X_1;\rho)]$. 
    The term $\mathbb{E}_n[\boldsymbol{f}_1(\rho)]$ and $\mathbb{G}_n[\boldsymbol{f}_1(\rho)]$ are similarly defined.
    \end{lem}
    \begin{proof}[Proof of Lemma \ref{lem: delta1}]
        By the definition of $\boldsymbol{\Delta}_1$, we have that 
        \begin{align*}
            \boldsymbol{\Delta}_1 &= \sqrt{n}\cdot\boldsymbol{G}^{-1}(\widehat{\boldsymbol{\tau}} - \boldsymbol{\tau}) = \sqrt{n}\cdot\big(\EB_n [\boldsymbol{f}_1(\widehat{\rho}\,)]-\EB [\boldsymbol{f}_1(\rho)]\big) = \sG_n[\boldsymbol{f}_1(\widehat{\rho}\,)]+\sqrt{n}\cdot\EB[\boldsymbol{f}_1(\widehat{\rho}\,)-\boldsymbol{f}_1(\rho)].
        \end{align*}
        By using Lemma \ref{lem:19.24}, we further have that
        \begin{align*}
            \boldsymbol{\Delta}_1 = \sG_n[\boldsymbol{f}_1(\rho)]+\sqrt{n}\cdot\EB[\boldsymbol{f}_1(\widehat{\rho}\,)-\boldsymbol{f}_1(\rho)]+o_P(1).
        \end{align*}
        This finishes the proof of Lemma \ref{lem: delta1}.
    \end{proof}

    \vspace{3mm}
    \noindent
    \textbf{Analysis of the term $\boldsymbol{\Delta_2}$.}
    The derivation of asymptotic result for $\boldsymbol{\Delta}_2$ is similar with Lemma \ref{lem: delta1}. 
    \begin{lem}
    \label{lem: delta2}
    We denote $\boldsymbol{f}_2(X_i;\rho), \boldsymbol{h}_0(X_i;\rho):\mathcal{X}\times\mathbb{R}^{S\times A\times Z}\mapsto\mathbb{R}^{d\times d}$ as 
    \begin{align*}
        \boldsymbol{f}_2(X_i;\rho)=\boldsymbol{G}^{-1}\boldsymbol{h}_0(X_i;\rho)\EB [\boldsymbol{f}_1(\rho)],\quad \boldsymbol{h}_0(X_i;\rho)=\boldsymbol{\phi}_{\rho}(z_i)(\boldsymbol{\phi}_{\rho}(z_i)-\gamma\boldsymbol{\phi}_{\pi^e}(s_i^{\prime}))^\top,
    \end{align*}
    Then the term $\boldsymbol{\Delta}_2$ satisfies that
    \begin{align*}
        \boldsymbol{\Delta}_2 = -\sG_n[\boldsymbol{f}_2(\rho)] -\sqrt{n}\cdot\big(\EB [\boldsymbol{f}_2(\widehat{\rho}\,)]-\EB [\boldsymbol{f}_2(\rho)]\big)+o_P(1),
    \end{align*}
    where $\EB [\boldsymbol{f}_2(\rho)]=\boldsymbol{G}^{-1}\EB [\boldsymbol{h}_0(\rho)] \EB [\boldsymbol{f}_1(\rho)]$ and $\EB [\boldsymbol{h}_0(\rho)] = \mathbb{E}_{X_1}[\boldsymbol{h}_0(X_1;\rho)]$.
    The term $\EB \boldsymbol{f}_2(\widehat{\rho}\,)$ is defined similarly.
    \end{lem}
    \begin{proof}[Proof of Lemma \ref{lem: delta2}]
    By the definition of $\boldsymbol{\Delta}_2$, we have that 
    \begin{align}
        \boldsymbol{\Delta}_2 &= \sqrt{n}\cdot(\widehat{\boldsymbol{G}}^{-1}-\boldsymbol{G}^{-1})\boldsymbol{\tau}\notag\\
        &=\sqrt{n}\cdot\boldsymbol{G}^{-1}(\boldsymbol{G}-\widehat{\boldsymbol{G}})\widehat{\boldsymbol{G}}^{-1}\boldsymbol{\tau}\notag\\
        &= \sqrt{n}\cdot\boldsymbol{G}^{-1}(\boldsymbol{G}-\widehat{\boldsymbol{G}})\boldsymbol{G}^{-1}\boldsymbol{\tau}+\sqrt{n}\cdot\boldsymbol{G}^{-1}(\boldsymbol{G}-\widehat{\boldsymbol{G}})(\widehat{\boldsymbol{G}}^{-1}-\boldsymbol{G}^{-1})\boldsymbol{\tau}\notag\\
        & =\sqrt{n}\cdot\boldsymbol{G}^{-1}(\boldsymbol{G}-\widehat{\boldsymbol{G}}) \EB[\boldsymbol{f}_1(\rho)]+\underbrace{\sqrt{n}\cdot\boldsymbol{G}^{-1}(\boldsymbol{G}-\widehat{\boldsymbol{G}})(\widehat{\boldsymbol{G}}^{-1}-\boldsymbol{G}^{-1})\boldsymbol{G}\EB [\boldsymbol{f}_1(\rho)]}_{\boldsymbol{\Delta}_{4}}\label{eq: proof lemma delta 2 1}.
    \end{align}
    By definition of $\boldsymbol{G}$ and $\widehat{\boldsymbol{G}}$, we have that 
    \begin{align}\label{eq: proof lemma delta 2 2}
        \boldsymbol{G}-\widehat{\boldsymbol{G}}=\EB [\boldsymbol{h}_0(\rho)]-\EB_n [\boldsymbol{h}_0(\widehat{\rho}\,)].
    \end{align}
    Therefore, by combining \eqref{eq: proof lemma delta 2 1} and \eqref{eq: proof lemma delta 2 2}, we have that 
    \begin{align*}
        \boldsymbol{\Delta}_2 = - \boldsymbol{G}^{-1}\big(\sG_n[\boldsymbol{h}_0(\widehat{\rho}\,)]+\sqrt{n}\cdot\EB[\boldsymbol{h}_0(\widehat{\rho}\,)-\boldsymbol{h}_0(\rho)]\big) \EB[\boldsymbol{f}_1(\rho)]+\boldsymbol{\Delta}_{4}.
    \end{align*}
    Again, by applying Lemma \ref{lem:19.24}, using the definition of $\boldsymbol{f}_2$, we further have that
    \begin{align}
        \boldsymbol{\Delta}_2 &= - \boldsymbol{G}^{-1}\big(\sG_n[\boldsymbol{h}_0(\rho)]+\sqrt{n}\cdot\EB[\boldsymbol{h}_0(\widehat{\rho}\,)-\boldsymbol{h}_0(\rho)]\big) \EB [\boldsymbol{f}_1(\rho)] + \boldsymbol{\Delta}_{4} + o_P(1)\notag\\
        &=-\sG_n [\boldsymbol{f}_2(\rho)]-\sqrt{n}\cdot\EB [\boldsymbol{f}_2(\widehat{\rho}\,)-\EB \boldsymbol{f}_2(\rho)]+\boldsymbol{\Delta}_{4}+o_P(1).
    \end{align}
    Now by applying the multidimensional central limit theorem (CLT), we can infer that
    \begin{align*}
        \sqrt{n}\cdot\big(\vect(\widehat{\boldsymbol{G}})-\vect(\boldsymbol{G})\big)=O_P(1).
    \end{align*}
    where $\vect(\boldsymbol{G})$ refers to the vector formed by the elements of $\boldsymbol{G}$ and $\vect(\widehat{\boldsymbol{G}})$ is similarly defined.
    Furthermore, by applying Lemma \ref{lem:G_consistent}, we have that
    \begin{align*}
        \|\widehat{\boldsymbol{G}}^{-1}-\boldsymbol{G}^{-1}\|=o_P(1).
    \end{align*}
    Therefore, we obtain that the term $\boldsymbol{\Delta}_{4}$ satisfies that $\boldsymbol{\Delta}_{4} = o_P(1)$.
    Consequently, we can conclude that 
    \begin{align*}
        \boldsymbol{\Delta}_2 = -\sG_n [\boldsymbol{f}_2(\rho)]-\sqrt{n}\cdot\big(\EB[ \boldsymbol{f}_2(\widehat{\rho}\,)]-\EB[\boldsymbol{f}_2(\rho)]\big)+o_P(1).
    \end{align*}
    This finishes the proof of Lemma \ref{lem: delta2}.
    \end{proof}

    \vspace{3mm}
    \noindent
    \textbf{Analysis of the term $\boldsymbol{\Delta_3}$.} One can see that the term $\boldsymbol{\Delta_3}$ can be negelected, which is the following lemma.
    \begin{lem}
    \label{lem: delta3}
        The term $\boldsymbol{\Delta}_3$ satisfies that $\boldsymbol{\Delta}_3=o_P(1)$.
    \end{lem}
    \begin{proof}[Proof of Lemma \ref{lem: delta3}]
    By definitoin of $\boldsymbol{\Delta}_3$, we have that
    \begin{align*}
        \boldsymbol{\Delta}_3 = (\widehat{\boldsymbol{G}}^{-1}-\boldsymbol{G}^{-1})\boldsymbol{G}\boldsymbol{\Delta}_1.
    \end{align*}
    By applying Lemma \ref{lem:G_consistent} and Lemma \ref{lem: delta1}, we can easily find that $\boldsymbol{\Delta_3} = o_P(1)$.
    \end{proof}

    \vspace{3mm}
    \noindent
    \textbf{Combining the analysis of the term $\boldsymbol{\Delta}_1$, $\boldsymbol{\Delta}_2$, and $\boldsymbol{\Delta}_3$.}
    Now we are ready to prove the main result. 
    By the decomposition of $\sqrt{n}\cdot(\widehat{\boldsymbol{\theta}}-\boldsymbol{\theta})$ in \eqref{eq: decomposition asymptotic mdp} and Lemma~\ref{lem: delta1}, \ref{lem: delta2}, and \ref{lem: delta3}, we have that
    \begin{align*}
        \sqrt{n}\cdot(\widehat{\boldsymbol{\theta}}-\boldsymbol{\theta})=\sG_n [\boldsymbol{f}(\rho)]+\sqrt{n}\cdot\big(\EB [\boldsymbol{f}(\widehat{\rho}\,)]-\EB [\boldsymbol{f}(\rho)]\big)+o_P(1).
    \end{align*}
    By applying the Delta method, we have that 
    \begin{align*}
        \sqrt{n}\cdot(\widehat{\boldsymbol{\theta}}-\boldsymbol{\theta})=\sG_n [\boldsymbol{f}(\rho)]+\sqrt{n}\cdot\nabla_{\rho}\EB [\boldsymbol{f}(\rho)](\widehat{\rho}-\rho)+o_P(1).
    \end{align*}
    Recall the definition of $\boldsymbol{g}$ and $\boldsymbol{d}$ in Theorem \ref{thm: asymptotic mdp}, we have that 
    \begin{align*}
        \widehat{\rho}=\boldsymbol{d}\left(\EB_n [\boldsymbol{g}]\right),\quad \rho=\boldsymbol{d}(\EB [\boldsymbol{g}]).
    \end{align*}
    Therefore, by applying another Delta method, we can obtain that 
    \begin{align*}
        \sqrt{n}\cdot(\widehat{\boldsymbol{\theta}}-\boldsymbol{\theta})=\sG_n [\boldsymbol{f}(\rho)]+\nabla_{\rho}\EB [\boldsymbol{f}(\rho)]\cdot\nabla_{\EB [\boldsymbol{g}]}\boldsymbol{d}(\mathbb{E}[\boldsymbol{g}])\cdot\sG_n [\boldsymbol{g}]+o_P(1)=\sG_n \boldsymbol{h}(\rho) +o_P(1).
    \end{align*}
    This finishes the proof of Theorem \ref{thm: asymptotic mdp}.
\end{proof}

\section{Auxiliary Lemmas}\label{sec: aux lemma}

\begin{lem}[Evaluation Error]\label{lem: evaluation error}
    For any function $Q:\mathcal{S}\times\mathcal{A}\mapsto\mathbb{R}$, $V=\mathbb{E}_{a\sim \pi^e(\cdot|s)}[Q(s,a)]:\mathcal{S}\mapsto\mathbb{R}$, and for any initial state $s_0\in\mathcal{S}$, it holds that 
    \begin{align}\label{eq: evaluation error lemma}
        V(s_0) - V_{\pi^e}(s_0) = \frac{1}{1-\gamma}\cdot\mathbb{E}_{(s,a)\sim d_{\pi^e,s_0}}\left[Q(s,a) - R(s,a) - \gamma \mathbb{E}_{s^{\prime}\sim P(\cdot|s,a),a^{\prime}\sim \pi^e(\cdot|s^{\prime})}[V(s^{\prime})]\right].
    \end{align}
\end{lem}
\begin{proof}[Proof of Lemma \ref{lem: evaluation error}]
    We adapt the proof from \cite{xie2020q}. 
    Note that by our assumption in Section \ref{subsec: data and eval}, $\epsilon^{(t)}$ is independent of $(s^{(t)},a^{(t)})$ and is zero mean.
    This means that we can equivalently write \eqref{eq: V pi e} as
    \begin{align*}
        V_{\pi^e}(s_0) 
        = \mathbb{E}_{\pi^{e}}\left[\sum_{t=1}^{+\infty}\gamma^{t-1}R(s^{(t)},a^{(t)})\middle| s^{(1)}=s_0\right] = \frac{1}{1-\gamma}\mathbb{E}_{(s,a)\sim d_{\pi^e,s_0}}\left[R(s,a)\right].
    \end{align*}
    Therefore, to prove \eqref{eq: evaluation error lemma}, it suffices to prove that 
    \begin{align}\label{eq: evaluation error lemma 1}
        V_{\pi^e}(s_0) = \frac{1}{1-\gamma}\cdot\mathbb{E}_{(s,a)\sim d_{\pi^e,s_0}}\left[Q(s,a) - \gamma \mathbb{E}_{s^{\prime}\sim P(\cdot|s,a),a^{\prime}\sim \pi^e(\cdot|s^{\prime})}[V(s^{\prime})]\right].
    \end{align}
    To this end, consider the right hand side of \eqref{eq: evaluation error lemma 1}. By the definition of $V$, we have that 
    \begin{align}
        &\frac{1}{1-\gamma}\cdot\mathbb{E}_{(s,a)\sim d_{\pi^e,s_0}}\left[Q(s,a) - \gamma \mathbb{E}_{s^{\prime}\sim P(\cdot|s,a),a^{\prime}\sim \pi^e(\cdot|s^{\prime})}[V(s^{\prime})]\right]\notag\\
        &\quad\quad = \sum_{t=1}^{+\infty}\sum_{(s,a)\in\SM\times\AM}\gamma^{t-1}\mathbb{P}_{\pi^e}(s^{(t)}=s, a^{(t)} = a|s^{(1)} = s_0)Q(s,a) \notag\\
        &\quad\quad\quad\quad -\sum_{t=1}^{+\infty}\sum_{(s,a)\in\SM\times\AM}\gamma^{t}\mathbb{P}_{\pi^e}(s^{(t+1)}=s, a^{(t+1)} = a|s^{(1)} = s_0)Q(s,a) \notag\\
        &\quad\quad = \sum_{(s,a)\in\SM\times\AM}\mathbb{P}_{\pi^e}(s^{(1)}=s, a^{(1)} = a|s^{(1)} = s_0)Q(s,a) = V_{\pi^e}(s_0)
    \end{align}
    This finishes the proof of Lemma \ref{lem: evaluation error}.
\end{proof}

\begin{lem}[Concentration of Empirical Feature]\label{lem:concentration_empirical_feature}
    We have the following concentration bounds for the empirical feature $\boldsymbol{\phi}_{\widehat{\rho}}$. In particular, with probability at least $1-\delta$, the following two inequalities hold,
    \begin{align*}
        \frac{1}{n}\sum_{i=1}^n\|\boldsymbol{\phi}_{\widehat{\rho}}(z_i)-\boldsymbol{\phi}_{\rho}(z_i)\|_2\leq\sqrt{\frac{3d\log\left(d|\mathcal{S}||\mathcal{A}||\mathcal{Z}|/\delta\right)}{n}}+\frac{4\sqrt{d}|\mathcal{Z}|\log\left(d|\mathcal{S}||\mathcal{A}||\mathcal{Z}|/\delta\right)}{n},
    \end{align*}
\end{lem}
\begin{proof}[Proof of Lemma \ref{lem:concentration_empirical_feature}]
    Denote the $\sigma$-algebra $\mathcal{F}=\sigma(z_1,\cdots,z_n)$. 
    For any $(s,a,z)\in\mathcal{S}\times\mathcal{A}\times\mathcal{Z}$ and $1\leq k\leq d$ where $d$ is the dimension of the feature mapping $\boldsymbol{\phi}(s,a)$, we denote that
    \begin{align*}
        w_i^k=\frac{\boldsymbol{\phi}^{(k)}(s_i,a_i)\mathbbm{1}\{z_i=z\}}{\frac{1}{n}\sum_{j=1}^n\mathbbm{1}\{z_j=z\}},
    \end{align*}
    where for simplicity we omit the dependence of $w_i^k$ on $(s,a,z)$. 
    Here we use $\boldsymbol{\phi}^{(k)}$ to denote the $k$-th coordinate of the feature mapping $\boldsymbol{\phi}$.
    By the definition of $\widehat{\rho}(s,a,z)$ in \eqref{eq: hat rho}, it holds that
    \begin{align*}
        \boldsymbol{\phi}_{\widehat{\rho}}^{(k)}(z)=\frac{1}{n}\sum_{i=1}^n w_i^k,
    \end{align*}
    where similarly $\boldsymbol{\phi}^{(k)}_{\widehat{\rho}}$ denotes the $k$-th coordinate of $\boldsymbol{\phi}_{\widehat{\rho}}(z)$.
    Now we can derive concentration inequality for $w_i^k$ conditioning on the $\sigma$-algebra $\mathcal{F}$. We use $\mathbb{V}[\cdot]$ to denote variance operator. Specifically, we have that
    \begin{align*}
        \mathbb{E}[w_i^k|\mathcal{F}]=\frac{\mathbb{E}[\boldsymbol{\phi}^{(k)}(s_i,a_i)|\mathcal{F}]\mathbbm{1}\{z_i=z\}}{\frac{1}{n}\sum_{j=1}^n\mathbbm{1}\{z_j=z\}}=\frac{\boldsymbol{\phi}^{(k)}_{\rho}(z)\mathbbm{1}\{z_i=z\}}{\frac{1}{n}\sum_{j=1}^n\mathbbm{1}\{z_j=z\}},
    \end{align*}
    \begin{align*}
        \mathbb{V}[w_i^k|\mathcal{F}]=\frac{\mathbb{E}[(\boldsymbol{\phi}^{(k)}(s_i,a_i)-\boldsymbol{\phi}_{\rho}^{(k)}(z_i))^2|\mathcal{F}]\mathbbm{1}\{z_i=z\}}{\big(\frac{1}{n}\sum_{j=1}^n\mathbbm{1}\{z_j=z\}\big)^2},
    \end{align*}
    from which we know that 
    \begin{align*}
        \frac{1}{n}\sum_{i=1}^n\mathbb{E}[w_i^k|\mathcal{F}]=\boldsymbol{\phi}_{\rho}^{(k)}(z),\quad\frac{1}{n}\sum_{i=1}^n\mathbb{V}[w_i^k|\mathcal{F}]=\frac{\mathbb{E}[(\boldsymbol{\phi}^{(k)}(s_i,a_i)-\boldsymbol{\phi}_{\rho}^{(k)}(z_i))^2|\mathcal{F}]}{\frac{1}{n}\sum_{j=1}^n\mathbbm{1}\{z_j=z\}}.
    \end{align*}
    Moreover, since $\boldsymbol{\phi}^{(k)}(s_i,a_i)$ is upper bounded by $1$ due to Assumption \ref{assump: linear}, it holds that 
    \begin{align*}
        |w_i^k|\leq \frac{n}{\sum_{j=1}^n\mathbbm{1}\{z_j=z\}},\quad  \frac{1}{n}\sum_{i=1}^n\mathbb{V}[w_i^k|\mathcal{F}]\leq \frac{n}{\sum_{j=1}^n\mathbbm{1}\{z_j=z\}}.
    \end{align*}
   Thus by applying Bernstein inequality, it holds with probability at least $1-\delta$ that
    \begin{equation}
        \mathbb{P}\Bigg(\Big|\frac{1}{n}\sum_{i=1}^nw_i^k-\boldsymbol{\phi}_{\rho}^{(k)}(z)\Big|\leq\sqrt{\frac{3\log\left(1/\delta\right)}{\sum_{j=1}^n\mathbbm{1}\{z_j=z\}}}+\frac{4\log\left(1/\delta\right)}{\sum_{j=1}^n\mathbbm{1}\{z_j=z\}} \Bigg| \mathcal{F}\Bigg)\geq 1-\delta.
    \end{equation}
    Taking expectation with respect to $z_i,\cdots,z_n$, we conclude that for any $(s,a,z)\in\mathcal{S}\times\mathcal{A}\times\mathcal{Z}$ and $1\leq k\leq d$, with full probability at least $1-\delta$, it holds that 
    \begin{align*}
        \big|\boldsymbol{\phi}_{\widehat{\rho}}^{(k)}(z)-\boldsymbol{\phi}_{\rho}^{(k)}(z)\big|\leq\sqrt{\frac{3\log\left(1/\delta\right)}{\sum_{j=1}^n\mathbf{1}\{z_j=z\}}}+\frac{4\log\left(1/\delta\right)}{\sum_{j=1}^n\mathbf{1}\{z_j=z\}}.
    \end{align*}
    Taking a union bound over $\mathcal{S}\times\mathcal{A}\times\mathcal{Z}\times\{1,\cdots,d\}$, we have that with probability at least $1-\delta$, it holds for all $(s,a,z)\in\mathcal{S}\times\mathcal{A}\times\mathcal{Z}$ and $1\leq k\leq d$ that 
    \begin{align}
        \big|\boldsymbol{\phi}_{\widehat{\rho}}^{(k)}(z)-\boldsymbol{\phi}_{\rho}^{(k)}(z)\big|\leq\sqrt{\frac{3\log\left(d|\mathcal{S}||\mathcal{A}||\mathcal{Z}|/\delta\right)}{\sum_{j=1}^n\mathbbm{1}\{z_j=z\}}}+\frac{4\log\left(d|\mathcal{S}||\mathcal{A}||\mathcal{Z}|/\delta\right)}{\sum_{j=1}^n\mathbbm{1}\{z_j=z\}}.
    \end{align}
    Now we can derive our final results.
    To this end, we have that with probability at least $1-\delta$,
    \begin{align*}
        \frac{1}{n}\sum_{i=1}^n\|\boldsymbol{\phi}_{\widehat{\rho}}(z_i)-\boldsymbol{\phi}_{\rho}(z_i)\|_2&=\frac{1}{n}\sum_{z\in\mathcal{Z}}\bigg(\sum_{j=1}^n\mathbbm{1}\{z_j=z\}\bigg)\|\boldsymbol{\phi}_{\widehat{\rho}}(z_i)-\boldsymbol{\phi}_{\rho}(z_i)\|_2\\
        &=\frac{1}{n}\sum_{z\in\mathcal{Z}}\bigg(\sum_{j=1}^n\mathbbm{1}\{z_j=z\}\bigg)\sqrt{\sum_{k=1}^d\big|\boldsymbol{\phi}_{\widehat{\rho}}^{(k)}(z)-\boldsymbol{\phi}_{\rho}^{(k)}(z)\big|^2}\\
        &\leq \frac{\sqrt{d}}{n}\sum_{z\in\mathcal{Z}}\bigg(\sum_{j=1}^n\mathbbm{1}\{z_j=z\}\bigg)\left(\sqrt{\frac{3\log\left(d|\mathcal{S}||\mathcal{A}||\mathcal{Z}|/\delta\right)}{\sum_{j=1}^n\mathbbm{1}\{z_j=z\}}}+\frac{4\log\left(d|\mathcal{S}||\mathcal{A}||\mathcal{Z}|/\delta\right)}{\sum_{j=1}^n\mathbbm{1}\{z_j=z\}}\right)\\
        &=\sqrt{d}\sum_{z\in\mathcal{Z}}\sqrt{\sum_{j=1}^n\mathbbm{1}\{z_j=z\}}\sqrt{\frac{3\log\left(d|\mathcal{S}||\mathcal{A}||\mathcal{Z}|/\delta\right)}{n^2}}+\frac{4\sqrt{d}|\mathcal{Z}|\log\left(d|\mathcal{S}||\mathcal{A}||\mathcal{Z}|/\delta\right)}{n}\\
        &\leq \sqrt{d}\sqrt{\sum_{z\in\mathcal{Z}}\sum_{j=1}^n\mathbbm{1}\{z_j=z\}}\sqrt{\frac{3\log\left(d|\mathcal{S}||\mathcal{A}||\mathcal{Z}|/\delta\right)}{n^2}}+\frac{4\sqrt{d}|\mathcal{Z}|\log\left(d|\mathcal{S}||\mathcal{A}||\mathcal{Z}|/\delta\right)}{n}\\
        &=\sqrt{\frac{3d\log\left(d|\mathcal{S}||\mathcal{A}||\mathcal{Z}|/\delta\right)}{n}}+\frac{4\sqrt{d}|\mathcal{Z}|\log\left(d|\mathcal{S}||\mathcal{A}||\mathcal{Z}|/\delta\right)}{n}.\\
    \end{align*}
    Here the last inequality is by the Cauchy-Schwarz inequality. 
    This finishes the proof of Lemma \ref{lem:concentration_empirical_feature}.
\end{proof}

 \begin{lem}[Covariance Matrix Estimation]\label{lem:cov_estimation}
    It holds with probability at least $1-2\delta$ that 
    \begin{align*}
        \|\boldsymbol{\Sigma}^{1/2}\widehat{\boldsymbol{\Sigma}}^{-1}\boldsymbol{\Sigma}^{1/2}\|\leq \left(1-\sqrt{\frac{8\underline{\sigma}^{-2}d\log\left(d|\mathcal{S}||\mathcal{A}||\mathcal{Z}|/\delta\right)}{n}}+\frac{6\underline{\sigma}^{-1}d|\mathcal{Z}|\log\left(d|\mathcal{S}||\mathcal{A}||\mathcal{Z}|/\delta\right)}{n}\right)^{-1}.
    \end{align*}    
\end{lem}
\begin{proof}[Proof of Lemma \ref{lem:cov_estimation}]
    We first relate $\boldsymbol{\Sigma}^{1/2}\widehat{\boldsymbol{\Sigma}}^{-1}\boldsymbol{\Sigma}^{1/2}$ to the difference $\widehat{\boldsymbol{\Sigma}}-\boldsymbol{\Sigma}$.
    Note that
    \begin{align*}
        (\boldsymbol{\Sigma}^{1/2}\widehat{\boldsymbol{\Sigma}}^{-1}\boldsymbol{\Sigma}^{1/2})^{-1}&=\boldsymbol{\Sigma}^{-1/2}\widehat{\boldsymbol{\Sigma}}\boldsymbol{\Sigma}^{-1/2}\\
        &=\boldsymbol{\Sigma}^{-1/2}(\boldsymbol{\Sigma}+\widehat{\boldsymbol{\Sigma}}-\boldsymbol{\Sigma})\boldsymbol{\Sigma}^{-1/2}\\
        &=\boldsymbol{I}_d+\boldsymbol{\Sigma}^{-1/2}(\widehat{\boldsymbol{\Sigma}}-\boldsymbol{\Sigma})\boldsymbol{\Sigma}^{-1/2}.
    \end{align*}
    Therefore, we have that 
    \begin{align*}
        \|\boldsymbol{\Sigma}^{1/2}\widehat{\boldsymbol{\Sigma}}^{-1}\boldsymbol{\Sigma}^{1/2}\|&=\sigma_{\min}^{-1}(I_d+\boldsymbol{\Sigma}^{-1/2}(\widehat{\boldsymbol{\Sigma}}-\boldsymbol{\Sigma})\boldsymbol{\Sigma}^{-1/2})\leq (1-\|\boldsymbol{\Sigma}^{-1/2}(\widehat{\boldsymbol{\Sigma}}-\boldsymbol{\Sigma})\boldsymbol{\Sigma}^{-1/2}\|)^{-1}, 
    \end{align*}
    which means that it suffices to bound $\|\boldsymbol{\Sigma}^{-1/2}(\widehat{\boldsymbol{\Sigma}}-\boldsymbol{\Sigma})\boldsymbol{\Sigma}^{-1/2}\|$.
    We decompose it into two terms as 
    \begin{align*}
        \|\boldsymbol{\Sigma}^{-1/2}(\widehat{\boldsymbol{\Sigma}}-\boldsymbol{\Sigma})\boldsymbol{\Sigma}^{-1/2}\|\leq \underbrace{\|\boldsymbol{\Sigma}^{-1/2}(\widehat{\boldsymbol{\Sigma}}-\bar{\boldsymbol{\Sigma}})\boldsymbol{\Sigma}^{-1/2}\|}_{\text{(i)}}+\underbrace{\|\boldsymbol{\Sigma}^{-1/2}(\bar{\boldsymbol{\Sigma}}-\boldsymbol{\Sigma})\boldsymbol{\Sigma}^{-1/2}\|}_{\text{(ii)}},
    \end{align*}
    where the matrix $\bar{\boldsymbol{\Sigma}}$ is defined as 
    \begin{align*}
        \bar{\boldsymbol{\Sigma}} = \frac{1}{n}\sum_{i=1}^n\boldsymbol{\phi}_{\rho}(z_i)\boldsymbol{\phi}_{\rho}(z_i)^\top.
    \end{align*}
    In the sequal, we upper bound the term (i) and term (ii) respectively.

    \vspace{3mm}
    \noindent
    \textbf{Bound on term (i).} By the definition of $\widehat{\boldsymbol{\Sigma}}$ and $\bar{\boldsymbol{\Sigma}}$, we have that 
    \begin{align}
        \boldsymbol{\Sigma}^{-1/2}(\widehat{\boldsymbol{\Sigma}}-\bar{\boldsymbol{\Sigma}})\boldsymbol{\Sigma}^{-1/2}&=\frac{1}{n}\sum_{i=1}^n\big(\boldsymbol{\Sigma}^{-1/2}\boldsymbol{\phi}_{\widehat{\rho}}(z_i)\big)(\boldsymbol{\Sigma}^{-1/2}\boldsymbol{\phi}_{\widehat{\rho}}(z_i)\big)^\top-\frac{1}{n}\sum_{i=1}^n(\boldsymbol{\Sigma}^{-1/2}\boldsymbol{\phi}_{\rho}(z_i)\big)(\boldsymbol{\Sigma}^{-1/2}\boldsymbol{\phi}_{\rho}(z_i)\big)^\top\notag\\
        &=\frac{1}{n}\sum_{i=1}^n\big(\boldsymbol{\Sigma}^{-1/2}\boldsymbol{\phi}_{\widehat{\rho}}(z_i)-\boldsymbol{\Sigma}^{-1/2}\boldsymbol{\phi}_{\rho}(z_i)\big)\big(\boldsymbol{\Sigma}^{-1/2}\boldsymbol{\phi}_{\widehat{\rho}}(z_i)+\boldsymbol{\Sigma}^{-1/2}\boldsymbol{\phi}_{\rho}(z_i)\big)^\top\label{equ:lem_1_0}.        
    \end{align}
    We can show that the right hand side of \eqref{equ:lem_1_0} is bounded by
    \begin{equation}\label{equ:lem_1_1}
        \begin{aligned}
        &\left\|\frac{1}{n}\sum_{i=1}^n\big(\boldsymbol{\Sigma}^{-1/2}\boldsymbol{\phi}_{\widehat{\rho}}(z_i)-\boldsymbol{\Sigma}^{-1/2}\boldsymbol{\phi}_{\rho}(z_i)\big)\big(\boldsymbol{\Sigma}^{-1/2}\boldsymbol{\phi}_{\widehat{\rho}}(z_i)+\boldsymbol{\Sigma}^{-1/2}\boldsymbol{\phi}_{\rho}(z_i)\big)^\top\right\|\\
        &\quad\quad\leq \frac{2\underline{\sigma}^{-1/2}}{n}\sum_{i=1}^n\left\|\boldsymbol{\Sigma}^{-1/2}\left(\boldsymbol{\phi}_{\widehat{\rho}}(z_i)-\boldsymbol{\phi}_{\rho}(z_i)\right)\right\|^2_2\leq \frac{2\underline{\sigma}^{-1}}{n} \sum_{i=1}^n\|\boldsymbol{\phi}_{\widehat{\rho}}(z_i)-\boldsymbol{\phi}_{\rho}(z_i)\|_2^2.
        \end{aligned}
    \end{equation}
    Invoking Lemma \ref{lem:concentration_empirical_feature}, combining Equation \eqref{equ:lem_1_0} and \eqref{equ:lem_1_1}, we conclude that with probability at least $1-\delta/2$, 
    \begin{align}
        \underbrace{\|\boldsymbol{\Sigma}^{-1/2}(\widehat{\boldsymbol{\Sigma}}-\bar{\boldsymbol{\Sigma}})\boldsymbol{\Sigma}^{-1/2}\|}_{\text{(i)}}
        &\leq \sqrt{\frac{3\underline{\sigma}^{-2}d\log\left(d|\mathcal{S}||\mathcal{A}||\mathcal{Z}|/\delta\right)}{n}}+\frac{4\underline{\sigma}^{-1}d|\mathcal{Z}|\log\left(d|\mathcal{S}||\mathcal{A}||\mathcal{Z}|/\delta\right)}{n}.
    \end{align}

    \vspace{3mm}
    \noindent
    \textbf{Bound on term (ii).} By the definition of $\bar{\boldsymbol{\Sigma}}$ and $\boldsymbol{\Sigma}$, we have that 
    \begin{equation}\label{equ:lem_2_0}
        \boldsymbol{\Sigma}^{-1/2}(\bar{\boldsymbol{\Sigma}}-\boldsymbol{\Sigma})\boldsymbol{\Sigma}^{-1/2}=\frac{1}{n}\sum_{i=1}^n\big(\boldsymbol{\Sigma}^{-1/2}\boldsymbol{\phi}_{\rho}(z_i)\big)\big(\boldsymbol{\Sigma}^{-1/2}\boldsymbol{\phi}_{\rho}(z_i)\big)^\top-\mathbb{E}\Big[\big(\boldsymbol{\Sigma}^{-1/2}\boldsymbol{\phi}_{\rho}(z_i)\big)\big(\boldsymbol{\Sigma}^{-1/2}\boldsymbol{\phi}_{\rho}(z_i)\big)^\top\Big].
    \end{equation}
    For notational simplicity, we denote the random matrix $\boldsymbol{Z}_i$ by 
    \begin{align*}
        \boldsymbol{Z}_i=\big(\boldsymbol{\Sigma}^{-1/2}\boldsymbol{\phi}_{\rho}(z_i)\big)\big(\boldsymbol{\Sigma}^{-1/2}\boldsymbol{\phi}_{\rho}(z_i)\big)^\top-\mathbb{E}\Big[\big(\boldsymbol{\Sigma}^{-1/2}\boldsymbol{\phi}_{\rho}(z_i)\big)\big(\boldsymbol{\Sigma}^{-1/2}\boldsymbol{\phi}_{\rho}(z_i)\big)^\top\Big], 
    \end{align*}
    We can find that $\mathbb{E}[\boldsymbol{Z}_i]=0$ and the operator norm of $\boldsymbol{Z}_i$ is bounded by
    \begin{align*}
        \|\boldsymbol{Z}_i\|\leq \max\left\{\|\boldsymbol{\Sigma}^{-1/2}\boldsymbol{\phi}_{\rho}(z_i)\|^2,1\right\}\leq \underline{\sigma}^{-1}.
    \end{align*}
    Moreover, the second moment of $\boldsymbol{Z}_i$ is given by 
    \begin{align*}
        \mathbb{E}[\boldsymbol{Z}_i^2]=\mathbb{E}\bigg[\Big(\big(\boldsymbol{\Sigma}^{-1/2}\boldsymbol{\phi}_{\rho}(z_i)\big)\big(\boldsymbol{\Sigma}^{-1/2}\boldsymbol{\phi}_{\rho}(z_i)\big)^\top\Big)^2\bigg]-\bigg(\mathbb{E}\Big[\big(\boldsymbol{\Sigma}^{-1/2}\boldsymbol{\phi}_{\rho}(z_i)\big)\big(\boldsymbol{\Sigma}^{-1/2}\boldsymbol{\phi}_{\rho}(z_i)\big)^\top\Big]\bigg)^2,
    \end{align*}
    which implies that the operator norm and the trace of $\mathbb{E}[\boldsymbol{Z}_i^2]$ can be bounded as
    \begin{align*}
        \|\mathbb{E}[\boldsymbol{Z}_i^2]\|\leq \Tr(\mathbb{E}[\boldsymbol{Z}_i^2])\leq d\cdot\left\|\mathbb{E}\bigg[\Big(\big(\boldsymbol{\Sigma}^{-1/2}\boldsymbol{\phi}_{\rho}(z_i)\big)\big(\boldsymbol{\Sigma}^{-1/2}\boldsymbol{\phi}_{\rho}(z_i)\big)^\top\Big)^2\bigg]\right\|\leq \underline{\sigma}^{-2}d.
    \end{align*}
    Now invoking Lemma \ref{lem:matrix_bernstein} and \eqref{equ:lem_2_0}, we conclude that with probability at least $1-\delta/2$, it holds that 
    \begin{align}
        \underbrace{\|\boldsymbol{\Sigma}^{-1/2}(\bar{\boldsymbol{\Sigma}}-\boldsymbol{\Sigma})\boldsymbol{\Sigma}^{-1/2}\|}_{\text{(ii)}}\leq \sqrt{\frac{4\underline{\sigma}^{-2}\log\left(d/\delta\right)}{n}}+\frac{2\underline{\sigma}^{-1}\log\left(d/\delta\right)}{3n}.
    \end{align}

    \vspace{3mm}
    \noindent
    \textbf{Combining bounds on term (i) and term (ii).}
    Finally, with probability at least $1-\delta$, it holds that
    \begin{align}
        \|\boldsymbol{\Sigma}^{1/2}\widehat{\boldsymbol{\Sigma}}^{-1}\boldsymbol{\Sigma}^{1/2}\|\leq \left(1-\sqrt{\frac{8\underline{\sigma}^{-2}d\log\left(d|\mathcal{S}||\mathcal{A}||\mathcal{Z}|/\delta\right)}{n}}+\frac{6\underline{\sigma}^{-1}d|\mathcal{Z}|\log\left(d|\mathcal{S}||\mathcal{A}||\mathcal{Z}|/\delta\right)}{n}\right)^{-1}.
    \end{align}
    This finishes the proof of Lemma \ref{lem:cov_estimation}.
\end{proof}

\begin{lem}[Matrix Bernstein Inequality \citep{hsu2012tailmatrix}]\label{lem:matrix_bernstein}
    Let $X$ be a random matrix, and $r>0, v>0$, and $k>0$ be such that, almost surely, $\mathbb{E}[X]=0$, $\sigma_{\max }(X) \leq r$, $\sigma_{\max}(\mathbb{E}[X^{2}]) \leq v$, $ \Tr(\mathbb{E}[X^{2}]) \leq v k$.
    Then if $X_{1}, X_{2}, \ldots, X_{n}$ are independent copies of $X$, then for any $t>0$,
    $$
        \mathbb{P}\Bigg(\sigma_{\max}\bigg(\frac{1}{n} \sum_{i=1}^{n} X_{i}\bigg)>\sqrt{\frac{2 v t}{n}}+\frac{r t}{3 n}\Bigg) \leq k t(e^{t}-t-1)^{-1} .
    $$
If $t \geq 2.6$, then it holds that $t\left(e^{t}-t-1\right)^{-1} \leq e^{-t / 2}$.
\end{lem}
\begin{proof}[Proof of Lemma \ref{lem:matrix_bernstein}]
    We refer to \cite{hsu2012tailmatrix} for a detailed proof Lemma \ref{lem:matrix_bernstein}.
\end{proof}

\begin{lem}\label{lem:G_consistent} For $\boldsymbol{G}$ and $\widehat{\boldsymbol{G}}$ defined in Appendix \ref{subsec: proof asymptotic mdp}, it holds that $\|\widehat{\boldsymbol{G}}^{-1}-\boldsymbol{G}^{-1}\|=o_P(1)$.
    \end{lem}
\begin{proof}[Proof of Lemma \ref{lem:G_consistent}]
    We first prove that $\|\widehat{\boldsymbol{G}}-\boldsymbol{G}\|=o_P(1)$. 
    To this end, let's denote another matrix $\widetilde{\boldsymbol{G}}$ as 
    \begin{align*}
        \widetilde{\boldsymbol{G}}=\frac{1}{n}\sum_{i=1}^n \boldsymbol{\phi}_{\rho}(z_i)(\boldsymbol{\phi}_{\rho}(z_i)-\gamma\boldsymbol{\phi}^\pi(s_i'))^\top,
    \end{align*}    
    which satisfies that $\EB[\widetilde{\boldsymbol{G}}]=\boldsymbol{G}$. 
    By the Strong Law of Large Number (SLLN), we have that $\|\widetilde{\boldsymbol{G}}-\boldsymbol{G}\|=o_P(1)$. 
    Noting that $\boldsymbol{\phi}_{\widehat{\rho}}(z)=\boldsymbol{\phi}_{\rho}(z)+\boldsymbol{\phi}_{\Delta}(z)$,
    where $\Delta(s,a,z)=\widehat{\rho}(s,a|z)-\rho(s,a|z)$. 
    Thus we have that
    \begin{align}\label{eq: proof G consistent 1}
        \widehat{\boldsymbol{G}}-\widetilde{\boldsymbol{G}}=\frac{1}{n}\sum_{i=1}^n\boldsymbol{\phi}_\rho(z_i)\boldsymbol{\phi}_{\Delta}(z_i)^\top+\frac{1}{n}\sum_{i=1}^n\boldsymbol{\phi}_{\Delta}(z_i)(\boldsymbol{\phi}_\rho(z_i)+\boldsymbol{\phi}_{\Delta}(z_i)-\gamma\boldsymbol{\phi}^\pi(s_i'))^\top.
    \end{align}
    For the first term in \eqref{eq: proof G consistent 1}, we have that
    \begin{align*}
        \left\|\frac{1}{n}\sum_{i=1}^n\boldsymbol{\phi}_\rho(z_i)\boldsymbol{\phi}_{\Delta}(z_i)^\top\right\|\le\max_{z\in\mathcal{Z}}\left\|\boldsymbol{\phi}_\rho(z)\boldsymbol{\phi}_{\Delta}(z)^\top\right\|=\max_{z\in\mathcal{Z}} \|\boldsymbol{\phi}_\rho(z)\|_2\cdot\|\boldsymbol{\phi}_{\Delta}(z)\|_2\leq \|\boldsymbol{\phi}_{\Delta}(z)\|_2=o_P(1),
    \end{align*}
    since $\widehat{\rho}$ is consistent to $\rho$.
    Similarly, for the second term in \eqref{eq: proof G consistent 1}, we can prove that its operation norm is also $o_P(1)$. 
    Therefore, we have that $\|\widehat{\boldsymbol{G}}-\widetilde{\boldsymbol{G}}\|=o_P(1)$ and consequently,
    \begin{align*}
        \|\widehat{\boldsymbol{G}}-\boldsymbol{G}\|\leq \|\widehat{\boldsymbol{G}}-\widetilde{\boldsymbol{G}}\|+\|\widetilde{\boldsymbol{G}} - \boldsymbol{G}\| = o_P(1).
    \end{align*}
    Finally, notice that $\widehat{\boldsymbol{G}}^{-1}-\boldsymbol{G}^{-1}=\widehat{\boldsymbol{G}}^{-1}(\boldsymbol{G}-\widehat{\boldsymbol{G}})\boldsymbol{G}^{-1}$. 
    Thus $\|\widehat{\boldsymbol{G}}^{-1}-\boldsymbol{G}^{-1}\|_{2}=o_P(1)$, finishing the proof.
\end{proof}

\begin{lem}[Lemma 19.24 in \cite{van2000asymptotic}]
    \label{lem:19.24}
    Suppose $\mathcal{F}$ is a P-Donsker class of measurable functions and $\widehat{f}_n$ is a sequence of random functions that take their values in $\mathcal{F}$ such that $\int(\widehat{f}_n(x)-f_0(x))^2 dP(x)$ converges in probability to $0$ for some $f_0\in L_2(P)$. Then $\sG_n[\widehat{f}_n-f_0]=o_P(1)$.
\end{lem}
\begin{proof}[Proof of Lemma \ref{lem:19.24}]
    We refer to Lemma 19.24 in \cite{van2000asymptotic} for a detailed proof.
\end{proof}

\end{document}